\documentclass[twoside]{article}
\usepackage{ecj,palatino,epsfig,latexsym,natbib}

\usepackage{graphicx}
\usepackage{placeins}

\usepackage{amsmath}
\usepackage{wasysym}
\usepackage{amssymb}
\usepackage{amsthm}
\usepackage{array}
\newcolumntype{L}[1]{>{\raggedright\let\newline\\\arraybackslash\hspace{0pt}}m{#1}}
\newcolumntype{C}[1]{>{\centering\let\newline\\\arraybackslash\hspace{0pt}}m{#1}}
\newcolumntype{R}[1]{>{\raggedleft\let\newline\\\arraybackslash\hspace{0pt}}m{#1}}
\usepackage{booktabs}
\usepackage{tabularx}
\usepackage{color}
\usepackage{floatflt}
\usepackage[table]{xcolor}
\definecolor{Gray}{gray}{0.9}
\definecolor{White}{rgb}{1,1,1}
\definecolor{Gray}{gray}{0.9}
\definecolor{LightCyan}{rgb}{0.88,1,1}
\usepackage{soul}
\usepackage{sidecap}
\graphicspath{{./Figures/}}
\usepackage{hyperref}
\hypersetup{colorlinks=true,
citecolor=black,
linkcolor=black,
urlcolor=black}

\usepackage{algorithm,algorithmicx}
\usepackage[noend]{algpseudocode}

\algrenewcommand\algorithmicrequire{\textbf{Precondition:}}
\algrenewcommand\algorithmicensure{\textbf{Postcondition:}}
\newtheorem{lex}{Theorem}[section]
\newtheorem{defn}[lex]{Definition}

\newcommand{\edit}[1]{{\color{black} #1}}

\begin{document}

%\numberofauthors{3}
%
%\alignauthor
%William La Cava\titlenote{corresponding author} \\
%       \affaddr{Department of Mechanical and Industrial Engineering}\\
%       \affaddr{University of Massachusetts}\\
%       \affaddr{Amherst, MA 01003}\\
%       \email{wlacava@umass.edu}
%%%\alignauthor
%%%Thomas Helmuth\\
%%%       \affaddr{Department of Computer Science}\\
%%%       \affaddr{University of Massachusetts}\\
%%%       \affaddr{Amherst, MA 01003}\\
%%%       \email{thelmuth@umass.edu}
%\alignauthor
%Lee Spector\\
%       \affaddr{School of Cognitive Science}\\
%       \affaddr{Hampshire College}\\
%       \affaddr{Amherst, MA 01002}\\
%       \email{lspector@hampshire.edu}
%\and
%\alignauthor
%Kourosh Danai\\
%       \affaddr{Department of Mechanical and Industrial Engineering}\\
%       \affaddr{University of Massachusetts}\\
%       \affaddr{Amherst, MA 01003}\\
%       \email{danai@engin.umass.edu}
%}

%\maketitle
%\ead{danai@ecs.umass.edu}
%\address[mie]{D, , Amherst, MA 01003}
%\address[hc]{School of Cognitive Science, Hampshire College, Amherst MA 01002}
%
\ecjHeader{x}{x}{xxx-xxx}{201X}{$\epsilon$-Lexicase selection}{W. La Cava}
\title{\edit{A probabilistic and multi-objective analysis of lexicase selection and $\epsilon$-lexicase selection}}  

\author{\name{\bf William La Cava} \hfill \addr{lacava@upenn.edu}\\ 
        \addr{Institute for Bioinformatics, University of Pennsylvania, 
        Philadelphia, PA, 19104, USA}
\AND
       \name{\bf Thomas Helmuth} \hfill \addr{helmutht@hamilton.edu}\\
        \addr{Department of Computer Science, Hamilton College, Clinton, NY, 13323, USA}        
\AND
       \name{\bf Lee Spector} \hfill \addr{lspector@hampshire.edu}\\
        \addr{School of Cognitive Science, Hampshire College, 
        Amherst, MA, 01002, USA}
\AND
	   \name{\bf Jason H. Moore} \hfill \addr{jhmoore@upenn.edu}\\
	   \addr{Institute for Bioinformatics, University of Pennsylvania, 
        Philadelphia, PA, 19104, USA}
}

\maketitle
\maketitle
\begin{abstract}
Lexicase selection is a parent selection method that considers training cases individually, rather than in aggregate, when performing parent selection. Whereas previous work has demonstrated the ability of lexicase selection to solve difficult problems \edit{in program synthesis and symbolic regression}, the central goal of this paper is to develop the theoretical underpinnings that explain its performance. To this end, we derive an analytical formula that gives the expected probabilities of selection under lexicase selection, given a population and its behavior. In addition, we expand upon the relation of lexicase selection to many-objective optimization methods to describe the behavior of lexicase selection, which is to select individuals on the boundaries of Pareto fronts in high-dimensional space. We show analytically why lexicase selection performs more poorly for certain sizes of population and training cases, and show why it has been shown to perform more poorly in continuous error spaces. To address this last concern, we \edit{propose new variants of $\epsilon$-lexicase selection, a method that modifies} the pass condition in lexicase selection to allow near-elite individuals to pass cases, thereby improving selection performance with continuous errors. We show that $\epsilon$-lexicase outperforms several diversity-maintenance strategies on a number of real-world and synthetic regression problems.
\end{abstract}

\section{Introduction}
  
Evolutionary computation (EC) traditionally assigns scalar fitness values to candidate solutions to determine how to guide search. In the case of genetic programming (GP), this fitness value summarizes how closely, on average, the behavior of the candidate programs match the desired behavior. Take for example the task of symbolic regression, in which we attempt to find a model using a set of training examples, i.e. cases. A typical fitness measure is the mean squared error (MSE), which averages the squared differences between the model's outputs, $\hat{y}$, and the target outputs, $y$. The effect of this averaging is to reduce a rich set of information comparing the model's output and the desired output to a single scalar value. As noted by \cite{krawiec_behavioral_2016}, the relationship of $\hat{y}$ to $y$ can only be represented crudely by this fitness value. The fitness score thereby restricts the information conveyed to the search process about candidate programs relative to the description of their behavior available in the raw comparisons of the output to the target, information which could help guide the search~\citep{krawiec_behavioral_2014, krawiec_automatic_2015}. This observation has led to increased interest in the development of methods that can leverage the program outputs directly to drive search more effectively~\citep{vanneschi_survey_2014}.

In addition to reducing information, averaging test performance assumes all tests are equally informative, leading to the potential loss of individuals who perform poorly {\it on average} even if they are the best on a training case that is difficult for most of the population to solve. This is particularly relevant for problems that require different modes of behavior to produce an adequate solution to the problem~\citep{spector_assessment_2012}. The underlying assumption of traditional selection methods is that selection pressure should be applied evenly with respect to training cases. In practice, cases that comprise the problem are unlikely to be uniformly difficult. \edit{In GP, the difficulty of a training case can be thought of as the probability of an arbitrary program solving it. Under the assumption that arbitrary programs do not uniformly solve training instances, it is unlikely that training instances will be uniformly difficult for a population of GP programs}. As a result, the search is likely to benefit if it can take into account the difficulty of specific cases by recognizing individuals that perform well on harder parts of the problem. Underlying this last point is the assumption that GP solves problems by identifying, propagating and recombining partial solutions (i.e. building blocks) to the task at hand~\citep{poli_schema_1998}. As a result, a program that performs well on unique subsets of the problem may \edit{contain} a partial solution to our task. 

Several methods have been proposed to reward individuals with uniquely good training performance, such as implicit fitness sharing (IFS)~\citep{mckay_investigation_2001}, historically assessed hardness~\citep{klein_genetic_2008}, and co-solvability~\citep{schaefer_using_2010}, all of which assign greater weight to fitness cases that are judged to be more difficult in view of the population performance. Perhaps the most effective parent selection method designed to account for case hardness is lexicase selection~\citep{spector_assessment_2012}. In particular, ``global  pool,  uniform  random  sequence,  elitist  lexicase  selection"~\citep{spector_assessment_2012}, which we refer to simply as lexicase selection, has outperformed other similarly-motivated methods in recent studies~\citep{helmuth_solving_2014, helmuth_general_2015-1, liskowski_comparison_2015}. Despite these gains, it fails to produce such benefits when applied to continuous symbolic regression problems, due to its method of selecting individuals based on training case elitism. For this reason we recently proposed~\citep{la_cava_epsilon-lexicase_2016} modulating the case pass conditions in lexicase selection using an automatically defined $\epsilon$ threshold, allowing the benefits of lexicase selection to be achieved in continuous domains. 

To date, lexicase selection and $\epsilon$-lexicase selection have mostly been analyzed via empirical studies, rather than algorithmic analysis. In particular, previous work has not explicitly described the probabilities of selection under lexicase selection compared to other selection methods, nor how lexicase selection relates to the multi-objective literature. Therefore, the foremost purpose of this paper is to describe analytically how lexicase selection and $\epsilon$-lexicase selection operate on a given population compared to other approaches. With this in mind, in \S\ref{s:prob} we derive an equation that describes the expected probability of selection for individuals in a given population based on their behavior on the training cases, for all variants of lexicase selection described here. Then in \S\ref{s:mo}, we analyze lexicase and $\epsilon$-lexicase selection from a multi-objective viewpoint, in which we consider each training case to be an objective. We prove that individuals selected by lexicase selection exist at the boundaries of the Pareto front defined by the program error vectors. We show via an illustrative example population in \S\ref{s:ex} how the probabilities of selection differ under tournament, lexicase, and $\epsilon$-lexicase selection. 

The second purpose of this paper is to empirically assess the use of $\epsilon$-lexicase selection in \edit{the task of symbolic regression}. In \S\ref{s:eplex}, we define two new variants of $\epsilon$-lexicase selection: semi-dynamic and dynamic, which are shown to improve the method compared to the original static implementation. A set of experiments compares variants of $\epsilon$-lexicase selection to several existing selection techniques on a set of real world benchmark problems. The results show the ability of $\epsilon$-lexicase selection to improve the predictive accuracy of models on these problems. We examine in detail the diversity of programs during these runs, as well as the number of cases used in selection events to validate our hypothesis that $\epsilon$-lexicase selection allows for more cases to be used when selecting individuals compared to lexicase selection. Lastly, the time complexity of lexicase selection is experimentally analyzed as a function of population size.    

\section{Methods}
\subsection{Preliminaries}
In symbolic regression, we attempt to find a model $\hat{y}(\mathbf{x}): \mathbb{R}^d \rightarrow \mathbb{R}$ that maps variables to a target output using a set of $T$ training examples $\mathcal{T} = \{t_i = (y_i,\mathbf{x}_i)\}_{i=1}^T$, where $\mathbf{x}$ is a $d$-dimensional vector of variables, i.e. features, and $y$ is the desired output. We refer to elements of $\mathcal{T}$ as ``cases". GP poses the problem as
\begin{equation} \label{eq:gp}
\text{minimize} \hspace{0.5em} {f(n,\mathcal{T})} \hspace{1em} \text{subject to} \hspace{0.5em} {n \in \mathfrak{N}}
\end{equation}
where $\mathfrak{N}$ is the space of possible programs $n$ and $f$ denotes a minimized fitness function. GP attempts to solve the symbolic regression task by optimizing a population of $N$ programs $\mathcal{N} = \{n_i\}_{i=1}^N$, each of which encodes a model of the process and produces an estimate  $\hat{y}_t(n,\mathbf{x}_t): \mathbb{R}^d \rightarrow \mathbb{R}$ when evaluated on case $t$. We refer to $\hat{y}(n)$ as the {\it semantics} of program $n$, omitting $\mathbf{x}$ for brevity. We denote the squared differences between $\hat{y}$ and $y$ (i.e., the errors) as $e_t(n) = (y_t - \hat{y}_t(n))^2$.  We use $\mathbf{e}_t \in \mathbb{R}^N$ to refer to the errors of all programs in the population on training case $t$. The lowest error in $\mathbf{e}_t$ is referred to as $e^*_t$. 

A typical fitness measure ($f$) is the mean squared error, $\text{MSE}(n,\mathcal{T}) = \frac{1}{N} \sum_{t \in \mathcal{T}}{e_t(n)}$, which we use to compare our results in \S\ref{s:results}. For the purposes of our discussion, it is irrelevant whether the MSE or the mean absolute error, i.e. MAE$(n,\mathcal{T}) = \frac{1}{N} \sum_{t \in \mathcal{T}}{|y_t - \hat{y}_t(n)|}$, is used, and so we use MAE to simplify a few examples throughout the paper. With lexicase selection and its variants, $e(n)$ is used directly during selection rather than averaging over cases. Nevertheless, in keeping with the problem statement in Eqn.~\ref{eq:gp}, the final program returned in our experiments is that which minimizes the MSE.

%$\mathcal{T} = \{ (y_t,\mathbf{x}_t)\}_{t = 1}^N$, using e.g. the mean absolute error (MAE), which is quantified for individual program $i \in P$ as:
%\begin{equation}
%MSE(\mathcal{T}) = \frac{1}{N} \sum_{t \in \mathcal{T}}{|y_t - \hat{y}_t(\mathbf{x}_t)|} \label{eq:fit} 
%\end{equation}
%
%where $\mathbf{x} \in \mathbb{R}^D$ represents the variables or features, the target output is $y$ and $\hat{y}(i,\mathbf{x})$ is the program's output. As a result of the aggregation of the absolute error vector $e(i) = |y - \hat{y}(i,\mathbf{x})|$ in Eq.~(\ref{eq:fit}), the relationship of $\hat{y}$ to $y$ is represented crudely when choosing models to propagate.

\subsection{Lexicase Selection}\label{s:lex}
Lexicase selection is a parent selection technique based on lexicographic ordering of training (i.e. fitness) cases. The lexicase selection algorithm for a single selection event is presented in Algorithm~\ref{alg:lex}. 

\begin{algorithm}
\caption{{\bf Lexicase Selection} applied to individuals $n \in \mathcal{N}$ with errors $e_t(n)$ on training cases $t \in \mathcal{T}$. }\label{alg:lex}
\noindent{\footnotesize
\begin{tabularx}{\textwidth}{lX}
%\multicolumn{2}{c}{{\normalsize Algorithm 3.1: Lexicase Selection}} \\
\texttt{Selection}($\mathcal{N,T}$, \texttt{ns}) 	:						&	\\
\hspace{1em} $\mathcal{P} \leftarrow \emptyset$ & set of selected parents \\
\hspace{1em} \textbf{do} \texttt{ns} \textbf{times}: & \texttt{ns} is the number of selection events\\
\hspace{1em}  \hspace{1em} $\mathcal{P} \leftarrow \mathcal{P} \; \cup $ \texttt{GetParent}($\mathcal{N,T}$) & add selected program to $\mathcal{P}$ \\
\\
\texttt{GetParent}($\mathcal{N,T}$) 	:						&	\\
\hspace{1em}	$\mathcal{T}' \leftarrow \mathcal{T}$	&	training cases\\
\hspace{1em}	$S \leftarrow \mathcal{N}$	&	initial selection pool is the population\\
\hspace{1em}	\textbf{while} $|\mathcal{T}'| >0$ \textbf{and} $|\mathcal{S}|>1$:						&	main loop\\
\hspace{1em}\hspace{1em}	$t \leftarrow$ random choice from $\mathcal{T'}$ 	&	\hspace{1em}consider a random case\\
\hspace{1em}\hspace{1em}	\texttt{elite} $\leftarrow$ min $e_t(n)$ \textbf{for} $n \in \mathcal{S}$ 	&	\hspace{1em}determine elite fitness\\
\hspace{1em}\hspace{1em}	\textbf{for} $n \in \mathcal{S}$: 	&	\hspace{1em}reduce selection pool to elites\\
\hspace{1em}\hspace{1em}\hspace{1em}	 \textbf{if} $e_t(n)$ $\neq$ \texttt{elite} \textbf{then}	$\mathcal{S} \leftarrow \mathcal{S} \setminus \{n\}$			&	\\
\hspace{1em}\hspace{1em}	$\mathcal{T'} \leftarrow \mathcal{T'} \setminus \{t\}$ 				&	\hspace{1em}reduce remaining cases\\
\hspace{1em} \textbf{return} random choice from $\mathcal{S}$															&	return parent  
\end{tabularx}
}
\end{algorithm}

Algorithm~\ref{alg:lex} consists of just a few steps: 1) choosing a case, 2) filtering the selection pool based on that case, and 3) repeating until the cases are exhausted or the selection pool is reduced to one individual. If the selection pool is not reduced by the time each case has been considered, an individual is chosen randomly from the remaining pool, $\mathcal{S}$. 

Under lexicase selection, cases in $\mathcal{T}$ can be thought of as filters that reduce the selection pool to the individuals in the pool that are best on that case. Each parent selection event constructs a new path through these filters. We refer to individuals as ``passing" a case if they remain in the selection pool when the case is considered. The filtering strength of a case is affected by two main factors: its difficulty as defined by the number of individuals that the case filters from the selection pool, and its order in the selection event, which varies from selection to selection. These two factors are interwoven in lexicase selection because a case performs its filtering on a subset of the population created by a randomized sequence of cases that come before it. In other words, the difficulty of a case depends not only on the problem definition, but on the ordering of the case in the selection event, which is randomized for each selection.

\edit{The randomized case order and filtering} mechanisms allow selective pressure to continually shift to individuals that are elite on cases that are \edit{rarely} solved in $\mathcal{N}$. Because cases appear in various orderings during selection, there is selective pressure for individuals to solve unique {\it subsets} of cases. Lexicase selection thereby accounts for the difficulty of individual cases as well as the difficulty of solving arbitrarily-sized subsets of cases. This selection pressure leads to the preservation of high behavioral diversity during evolution~\citep{helmuth_effects_2016, la_cava_epsilon-lexicase_2016}. 

The worst-case complexity of selecting $N$ parents per generation with $|\mathcal{T}| = T$ test cases is $O(TN^2)$. This running time stems from the fact that to select a single individual, lexicase selection may have to consider the error value of every individual on every test case.
In contrast, tournament selection only needs to consider the precomputed fitnesses of a constant tournament size number of individuals; thus selecting a single parent can be done in constant time. Since errors need to be calculated and summed for every test case on every individual, tournament selection requires $O(TN)$ time to select $N$ parents.
Normally, due to differential performance across the population and due to lexicase selection's tendency to promote diversity, a lexicase selection event will use many fewer test cases than $T$; the selection pool typically winnows below $N$ as well, meaning the actual running time tends to be better than the worst-case complexity~\citep{helmuth_solving_2014, la_cava_epsilon-lexicase_2016}. 

We use an example population originally presented in~\citep{spector_assessment_2012} to illustrate some aspects of standard lexicase selection in the following sections. The population, shown in Table~\ref{tbl:ex}, consists of five individuals and four training cases with discrete errors. A graphical example of the filtering mechanism of selection is presented for this example in Figure~\ref{fig:lex_graph}. Each lexicase selection event can be visualized as a randomized depth-first pass through the training cases. Figure~\ref{fig:lex_graph} shows three example selection events resulting in the selection of different individuals. The population is winnowed at each case to the elites until single individuals, shown with diamond-shaped nodes, are selected. 

\begin{table}
\centering
\caption{Example population from original lexicase paper~\citep{spector_assessment_2012}. $P_{lex}$ and $P_t$ are the probabilities of selection under lexicase selection (Eqn.~\ref{eq:prob}) and tournament selection with tournament size 2 (Eqn.~\ref{eq:prob_t}), respectively.}\label{tbl:ex}
\begin{tabular}{l|cccc|c|r|rr}\toprule
Program & \multicolumn{4}{c}{Case Error} & Elite Cases & MAE & $P_{lex}$ & $P_{t}$\\
& $e_1$ & $e_2$ & $e_3$ & $e_4$ & \\ \midrule
$n_1$ & 2 & 2 & 4 & 2 & $\{t_2,t_4\}$ &	2.5		&	0.25 	& 	0.28	\\
$n_2$ & 1 & 2 & 4 & 3 & $\{t_2\}$		&	2.5		&	0.00	&	0.28	\\
$n_3$ & 2 & 2 & 3 & 4 & $\{t_2,t_3\}$ &	2.75	& 	0.33	&	0.12	\\
$n_4$ & 0 & 2 & 5 & 5 & $\{t_1,t_2\}$ &	3.0		& 	0.21	&	0.04	\\
$n_5$ & 0 & 3 & 5 & 2 & $\{t_1,t_4\}$ &	2.5		&	0.21	&	0.28	\\ \bottomrule
\end{tabular}
\end{table}

\begin{figure}[tb]
\centering
  \includegraphics[height = 0.3\textheight]{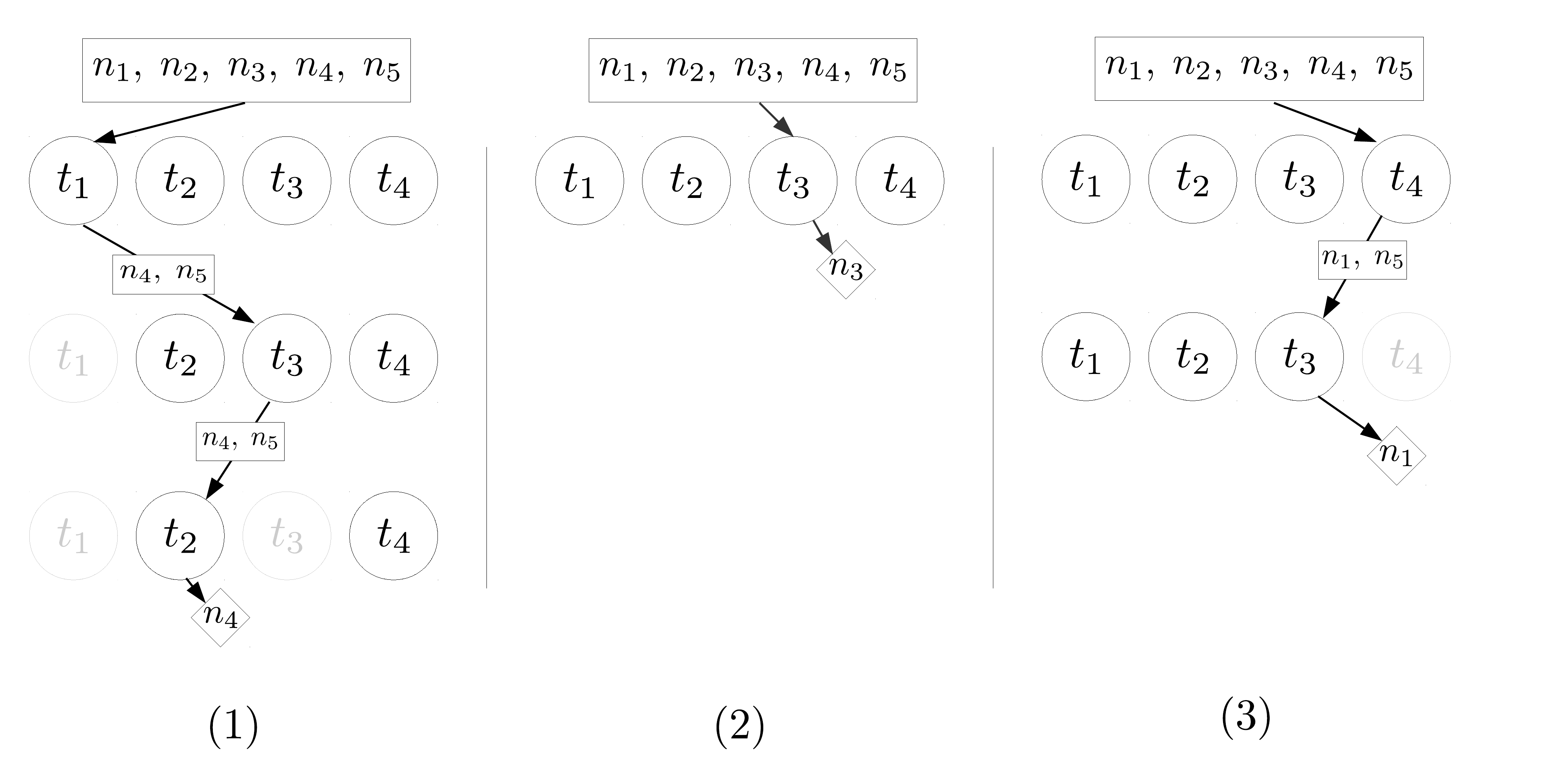}
  \caption{A graphical representation of three parent selections using lexicase selection on the population in Table~\ref{tbl:ex}. The arrows indicate different selection paths through the training cases in circles. The boxes indicate the selection pool after the case performs its filtering. The diamonds show the individual selected by each selection event. Training cases in gray indicate that they have already been traversed by the current parent selection process.}\label{fig:lex_graph}
\end{figure}

\subsection{{\large $\epsilon$}-Lexicase Selection}\label{s:eplex}

Lexicase selection has been shown to be effective in discrete error spaces, both for multi-modal problems~\citep{spector_assessment_2012} and for problems for which every case must be solved exactly to be considered a solution~\citep{helmuth_solving_2014, helmuth_general_2015-1}. In continuous error spaces, however, the requirement for individuals to be {\it exactly} equal to the elite error in the selection pool to pass a case turns out to be overly stringent~\citep{la_cava_epsilon-lexicase_2016}. In continuous error spaces and especially for symbolic regression with noisy datasets, it is unlikely for two individuals to have exactly the same error on any training case unless they are (or reduce to) equivalent models. As a result, lexicase selection is prone to conducting selection based on single cases, for which the selected individual satisfies $e_t \equiv e^*_t$, the minimum error on $t$ among $\mathcal{N}$. Selecting on single cases limits the ability of lexicase to leverage case information on subsets of test cases effectively, and can lead to poorer performance than traditional selection methods~\citep{la_cava_epsilon-lexicase_2016}. 

These observations led to the development of $\epsilon$-lexicase selection~\citep{la_cava_epsilon-lexicase_2016}, \edit{which modulates case filtering by calculating an $\epsilon$ threshold criteria for each training case.} Hand-tuned and automatic variants of $\epsilon$ were proposed and tested. The best performance was achieved by a 'parameter-less' version that defines $\epsilon$ according to the dispersion of errors in the population on each training case using the median absolute deviation statistic:  
\begin{equation}\label{eq:ep}
\epsilon_t = \lambda(\mathbf{e}_t) = \text{median}(|\mathbf{e}_t - \text{median}(\mathbf{e}_t)|) 
\end{equation}
\edit{Defining $\epsilon$ according to Eqn.~\ref{eq:ep} allows the threshold to conform to the performance of the population on each training case. As the performance on each training case improves across the population, $\epsilon$ shrinks, thereby modulating the selectivity of a case based on how difficult it is.} We choose the median absolute deviation in lieu of the standard deviation statistic for calculating $\epsilon$ because it is more robust to outliers~\citep{pham-gia_mean_2001}. 

We study three implementations of $\epsilon$-lexicase selection in this paper: static, which is the version originally proposed~\citep{la_cava_epsilon-lexicase_2016}; semi-dynamic, in which the elite error is defined relative to the current selection pool; and dynamic, in which both the elite error and $\epsilon$ are defined relative to the current selection pool. 

Static $\epsilon$-lexicase selection can be viewed as a preprocessing step added to lexicase selection in which the program errors are converted to pass/fail based on an $\epsilon$ threshold. This threshold is defined relative to $e^*_t$, the lowest error on test case $t$ over the entire population. We call this static $\epsilon$-lexicase selection because the elite error $e^*_t$ and $\epsilon$ are only calculated once per generation, instead of relative to the current selection pool, as described in Algorithm~\ref{alg:ep-lex-s}.

\begin{algorithm}
    \caption{{\bf Static $\epsilon$-Lexicase Selection} applied to individuals $n \in \mathcal{N}$ with errors $e_t(n)$ and minimum error $e_t^*$ on training cases $t \in \mathcal{T}$. $\lambda$ is the median absolute deviation function.}\label{alg:ep-lex-s}
\noindent{\footnotesize
\begin{tabularx}{\textwidth}{lX}
%\multicolumn{2}{c}{\normalsize Algorithm~\ref{alg:ep-lex-s}: {Static $\epsilon$-Lexicase Selection}} \\ 
\texttt{Selection}($\mathcal{N,T}$, \texttt{ns}) 	:						&	\\
\hspace{1em} $\mathcal{P} \leftarrow \emptyset$ & set of selected parents \\
    \hspace{1em}	$\epsilon \leftarrow \lambda$($\mathbf{e}_t$) \textbf{for} $t \in \mathcal{T}$	&	get $\epsilon$ for each case across population\\
\hspace{1em}	\textbf{for}	 $t \in \mathcal{T}$ \textbf{and} $n \in \mathcal{N}$:	&	define fitness $f$ using within-$\epsilon$ pass condition\\
\hspace{1em}\hspace{1em}	\textbf{if} $e_t(n) \leq e^*_t + \epsilon_t$ \textbf{then} $f_t(n) \leftarrow 0$ \\
\hspace{1em}\hspace{1em} \textbf{else} $f_t(n) \leftarrow 1$ \\ 
\hspace{1em} \textbf{do} \texttt{ns} \textbf{times}: & \texttt{ns} is the number of selection events\\
\hspace{1em}  \hspace{1em} $\mathcal{P} \leftarrow \mathcal{P} \; \cup $ \texttt{GetParent}($\mathcal{N,T}$,$f$) & add selected program to $\mathcal{P}$ \\
\\
\texttt{GetParent}($\mathcal{N,T}$, $f$) 	:						&	\\
\hspace{1em}	$\mathcal{T}' \leftarrow \mathcal{T}$	&	training cases\\
\hspace{1em}	$S \leftarrow \mathcal{N}$	&	initial selection pool is the population\\

\hspace{1em}	\textbf{while} $|\mathcal{T}'| >0$ \textbf{and} $|\mathcal{S}|>1$:						&	main loop\\
\hspace{1em}\hspace{1em}	$t \leftarrow$ random choice from $\mathcal{T'}$  	&	\hspace{1em}consider a random case\\
\hspace{1em}\hspace{1em}	\texttt{elite} $\leftarrow$ min $f_t(n)$ \textbf{for} $n \in \mathcal{S}$ 	&	\hspace{1em}determine elite fitness\\
%\hspace{1em}\hspace{1em}	$\mathcal{S} \leftarrow n \in \mathcal{S}$ if fitness($n$) $\leq$ \texttt{elite}	&	\hspace{1em}reduce selection pool\\
\hspace{1em}\hspace{1em}	\textbf{for} $n \in \mathcal{S}$: 	&	\hspace{1em}reduce selection pool\\
\hspace{1em}\hspace{1em}\hspace{1em}	 \textbf{if} $f_t(n) \neq$ \texttt{elite} \textbf{then}	$\mathcal{S} \leftarrow \mathcal{S} \setminus \{n\}$			&	\\
\hspace{1em}\hspace{1em}	$\mathcal{T'} \leftarrow \mathcal{T'} \setminus \{t\}$ 				&	\hspace{1em}reduce remaining cases\\
\hspace{1em} \textbf{return} random choice from $\mathcal{S}$															&	return parent  
\end{tabularx}
}
\end{algorithm}

Semi-dynamic $\epsilon$-lexicase selection differs from static $\epsilon$-lexicase selection in that the pass condition is defined relative to the best error {\it among the pool} rather than among the entire population $\mathcal{N}$. In this way it behaves more similarly to standard lexicase selection (Algorithm~\ref{alg:lex}), except that individuals are filtered out only if they have error more than $e^*_t + \epsilon_t$. It is defined in Algorithm~\ref{alg:ep-lex-sd}.

\begin{algorithm}
\caption{{\bf Semi-dynamic $\epsilon$-Lexicase Selection} applied to individuals $n \in \mathcal{N}$ with errors $e_t(n)$ on training cases $t \in \mathcal{T}$. $\lambda$ is the median absolute deviation function.}\label{alg:ep-lex-sd}
\noindent{\footnotesize
\begin{tabularx}{\textwidth}{lX}
%\multicolumn{2}{c}{\normalsize Algorithm~\ref{alg:ep-lex-sd}: Semi-dynamic $\epsilon$-Lexicase Selection} \\ 
\texttt{Selection}($\mathcal{N,T}$, \texttt{ns}) 	:						&	\\
\hspace{1em} $\mathcal{P} \leftarrow \emptyset$ & set of selected parents \\
\hspace{1em}	$\epsilon \leftarrow \lambda$($\mathbf{e}_t$) \textbf{for} $t \in \mathcal{T}$	&	get $\epsilon$ for each case across population\\
\hspace{1em} \textbf{do} \texttt{ns} \textbf{times}: & \texttt{ns} is the number of selection events\\
\hspace{1em}  \hspace{1em} $\mathcal{P} \leftarrow \mathcal{P} \; \cup $ \texttt{GetParent}($\mathcal{N,T},\epsilon$) & add selected program to $\mathcal{P}$ \\
\\
\texttt{GetParent}($\mathcal{N,T},\epsilon$) 	:						&	\\
\hspace{1em}	$\mathcal{T}' \leftarrow \mathcal{T}$	&	training cases\\
\hspace{1em}	$S \leftarrow \mathcal{N}$	&	initial selection pool is the population\\
\hspace{1em}	\textbf{while} $|\mathcal{T}'| >0$ \textbf{and} $|\mathcal{S}|>1$:						&	main loop\\
\hspace{1em}\hspace{1em}	$t$ $\leftarrow$ random choice from $\mathcal{T'}$ &	\hspace{1em}consider a random case\\
\hspace{1em}\hspace{1em}	\texttt{elite} $\leftarrow$ min $e_t(n)$ \textbf{for} $n \in \mathcal{S}$ 	&	\hspace{1em}determine elite fitness\\
%\hspace{1em}\hspace{1em}	$\mathcal{S} \leftarrow n \in \mathcal{S}$ if $e_t(n)$ $\leq$ \texttt{elite}$+\epsilon_{t}$	&	\hspace{1em}reduce selection pool\\
\hspace{1em}\hspace{1em}	\textbf{for} $n \in \mathcal{S}$: 	&	\hspace{1em}reduce selection pool\\
\hspace{1em}\hspace{1em}\hspace{1em}	 \textbf{if} $e_t(n)$ $>$ \texttt{elite}$+\epsilon_{t}$ \textbf{then}	$\mathcal{S} \leftarrow \mathcal{S} \setminus \{n\}$	 \\
\hspace{1em}\hspace{1em}	$\mathcal{T'} \leftarrow \mathcal{T'} \setminus \{t\}$				&	\hspace{1em}reduce remaining cases\\
\hspace{1em} \textbf{return} random choice from $\mathcal{S}$															&	return parent  
\end{tabularx}
}
\end{algorithm}

The final variant of $\epsilon$-lexicase selection is dynamic $\epsilon$-lexicase selection, in which both the error and $\epsilon$ are defined among the current selection pool. In this case, $\epsilon$ is defined as 
\begin{equation}\label{eq:epd}
\epsilon_t(\mathcal{S}) = \text{median}(|\mathbf{e}_t(\mathcal{S}) - \text{median}(\mathbf{e}_t(\mathcal{S}))|) = \lambda(\mathbf{e}_t(\mathcal{S}))
\end{equation}
where $\mathbf{e}_t(\mathcal{S})$ is the vector of errors for case $t$ among the current selection pool $\mathcal{S}$. The dynamic $\epsilon$-lexicase selection algorithm is presented in Algorithm~\ref{alg:ep-lex-d}.

\begin{algorithm}
\caption{{\bf Dynamic $\epsilon$-Lexicase Selection} applied to individuals $n \in \mathcal{N}$ with errors $e_t(n)$ on training cases $t \in \mathcal{T}$. $\lambda$ is the median absolute deviation function.}\label{alg:ep-lex-d}
\noindent{\footnotesize
\begin{tabularx}{\textwidth}{lX}
%\multicolumn{2}{c}{\normalsize Algorithm~\ref{alg:ep-lex-d}: Dynamic $\epsilon$-Lexicase Selection} \\ 
\texttt{Selection}($\mathcal{N,T}$, \texttt{ns}) 	:						&	\\
\hspace{1em} $\mathcal{P} \leftarrow \emptyset$ & set of selected parents \\
\hspace{1em} \textbf{do} \texttt{ns} \textbf{times}: & \texttt{ns} is the number of selection events\\
\hspace{1em}  \hspace{1em} $\mathcal{P} \leftarrow \mathcal{P} \; \cup $ \texttt{GetParent}($\mathcal{N,T}$) & add selected program to $\mathcal{P}$ \\
\\
\texttt{GetParent}($\mathcal{N,T}$) 	:						&	\\
\hspace{1em}	$T' \leftarrow \mathcal{T}$	&	training cases\\
\hspace{1em}	$S \leftarrow \mathcal{N}$	&	initial selection pool is the population\\
\hspace{1em}	\textbf{while} $|T'| >0$ \textbf{and} $|\mathcal{S}|>1$:						&	main loop\\
\hspace{1em}\hspace{1em}	$t$ $\leftarrow$ random choice from $\mathcal{T'}$  	&	\hspace{1em}consider a random case\\
\hspace{1em}\hspace{1em}	\texttt{elite} $\leftarrow$ min $e_t(n)$ \textbf{for} $n \in \mathcal{S}$ 	&	\hspace{1em}determine elite fitness\\
\hspace{1em}\hspace{1em}	$\epsilon_t \leftarrow \lambda(\mathbf{e}_t(\mathcal{S}))$	&	\hspace{1em}determine $\epsilon$ for case $t$\\
%\hspace{1em}\hspace{1em}	$\mathcal{S} \leftarrow n \in \mathcal{S}$ if $e_t(n)$ $\leq$ \texttt{elite}$+\epsilon_{t}$	&	\hspace{1em}reduce selection pool\\
\hspace{1em}\hspace{1em}	\textbf{for} $n \in \mathcal{S}$: 	&	\hspace{1em}reduce selection pool\\
\hspace{1em}\hspace{1em}\hspace{1em}	 \textbf{if} $e_t(n)$ $>$ \texttt{elite}$+\epsilon_{t}$ \textbf{then}	$\mathcal{S} \leftarrow \mathcal{S} \setminus \{n\}$	 \\
\hspace{1em}\hspace{1em}	$\mathcal{T'} \leftarrow \mathcal{T'} \setminus \{t\}$				&	\hspace{1em}reduce remaining cases\\
\hspace{1em} \textbf{return} random choice from $\mathcal{S}$															&	return parent  
\end{tabularx}
}
\end{algorithm}

Since calculating $\epsilon$ according to Eqn.~\ref{eq:ep} is $O(N)$ for a single test case, the three $\epsilon$-lexicase selection algorithms share a worst-case complexity with lexicase selection of $O(TN^2)$ to select $N$ parents. As discussed in \S\ref{s:lex}, these worst-case time complexities are rare, and empirical results have confirmed $\epsilon$-lexicase to run within the same time frame as tournament selection~\citep{la_cava_epsilon-lexicase_2016}. We assess the affect of population size on wall-clock times in \S\ref{s:exp}. %, which has a time complexity of $O(TN)$. 

\subsection{Related Work}\label{s:rw}

\edit{Lexicase selection belongs to a class of search drivers that incorporate a program's full semantics directly into the search process, and as such shares a general motivation with semantic GP methods. Geometric Semantic GP~\citep{moraglio_geometric_2012} uses a program's semantics in the variation step by defining mutation and crossover operators that make steps in semantic space. Intermediate program semantics can also be leveraged, as shown by Behavioral GP~\citep{krawiec_behavioral_2014}, which uses a program's execution trace to build an archive of program building blocks and learn intermediate concepts. Unlike lexicase selection, Behavioral GP generally exploits intermediate program semantics, rather than intermediate fitness cases, to guide search. These related semantic GP methods tend to use established selection methods while leveraging program semantics at other steps in the search process. }

% restructuring fitness
Instead of incorporating the full semantics, another option is to alter the fitness metric by weighting training cases based on population performance~\citep{mckay_investigation_2001}. In non-binary Implicit Fitness Sharing (IFS)~\citep{krawiec_implicit_2013}, for example, the fitness proportion of a case is scaled by the performance of other individuals on that case. Similarly, historically assessed hardness scales error on each training case by the success rate of the population~\citep{klein_genetic_2008}. \edit{These methods are able to capture a univariate notion of fitness case difficulty, but unlike lexicase selection, interactions between cases are not considered in estimating difficulty.}

% DOC and related methods
Discovery of objectives by clustering (DOC)~\citep{krawiec_automatic_2015} clusters training cases by population performance, and thereby reduces training cases into a set of objectives used in multi-objective optimization. Both IFS and DOC were outperformed by lexicase selection on program synthesis and boolean problems in previous studies~\citep{helmuth_general_2015-1,liskowski_comparison_2015}. \edit{More recently, \cite{liskowski_discovery_2017} proposed hybrid techniques that combine DOC and related objective derivation methods with $\epsilon$-lexicase selection, and found that this combination performed well on symbolic regression problems.} 

% Sampling methods
Other methods attempt to sample a subset of $\mathcal{T}$ to reduce computation time or improve performance, such as dynamic subset selection~\citep{gathercole_dynamic_1994}, interleaved sampling~\citep{goncalves_balancing_2013}, and co-evolved fitness predictors~\citep{schmidt_coevolution_2008}. Unlike these methods, lexicase selection begins each selection with the full set of training cases, and allows selection to adapt to program performance on them.
% Related selection methods
\edit{
Another approach to adjusting selection pressure based on population performance is to automatically tune tournament selection, a method which was investigated by \cite{xie_parent_2013}. In that work, tournament selection pressure was tuned to correspond to the distribution of fitness ranks in the population. 
}

%Multiobjective stuff
Although to an extent the ideas of multiobjective optimization apply to multiple training cases, they are qualitatively different and commonly operate at different scales. Symbolic regression often involves one or two objectives (e.g. accuracy and model conciseness) and hundreds or thousands of training cases. One example of using training cases explicitly as objectives occurs in~\cite{langdon_evolving_1995} in which small numbers of training cases (in this case 6) are used as multiple objectives in a Pareto selection scheme. Other multi-objective approaches such as NSGA-II~\citep{schoenauer_fast_2000}, SPEA2~\citep{zitzler_spea2:_2001} and ParetoGP~\citep{smits_pareto-front_2005} are commonly used with a small set of objectives in symbolic regression. \edit{The ``curse of dimensionality" makes the use of objectives at the scale of typical training case sizes problematic, since most individuals become nondominated. Scaling issues in many-objective optimization are reviewed by~\cite{ishibuchi_evolutionary_2008} and surveyed in~\cite{li_many-objective_2015}. Several methods have been proposed to deal with large numbers of objectives, including hypervolume-based methods such as HypE, reference point methods like NSGA-III, and problem decomposition methods like $\epsilon$-MOEA and MOEA/D~\citep{chand_evolutionary_2015}. \cite{li_empirical_2017} benchmarked several reference point methods on problems of up to 100 objectives, further shrinking the scalability gap.} The connection between lexicase selection and multi-objective methods is explored in depth in \S\ref{s:mo}.
%Pareto strength in SPEA2 promotes individuals based on how many individuals they dominate, and similarly lexicase selection increases the probability of selection for individuals who solve {\it more} cases and {\it harder} cases (i.e. cases that are not solved by other individuals) and decreases for individuals who solve {\it fewer} or {\it easier} cases. 

%epsilon stuff
The conversion of a model's real-valued fitness into discrete values based on an $\epsilon$ threshold has been explored in other research; for example, Novelty Search GP~\citep{martinez_searching_2013} uses a reduced error vector to define behavioral representation of individuals in the population. \cite{la_cava_epsilon-lexicase_2016} used it for the first time as a solution to applying lexicase selection effectively to regression, with static $\epsilon$-lexicase selection (Algorithm~\ref{alg:ep-lex-s}).

Recent work has empirically studied and extended lexicase selection.~\cite{helmuth_impact_2016} found that extreme selection events in lexicase selection were not central to its performance improvements and that lexicase selection could re-diversify less-diverse populations unlike tournament selection~\citep{helmuth_effects_2016}. A survival-based version of $\epsilon$-lexicase selection has also been proposed~\citep{la_cava_general_2017,la_cava_ensemble_2017} for maintaining uncorrelated populations in an ensemble learning context. 

\edit{
\section{Theoretical Analysis}
In the first half of this section ((\S\ref{s:prob}), we examine the probabilities of selection under lexicase selection. Our aims are to answer the following questions: First, what is the probability of an individual being selected by lexicase selection, given its performance in a population on a set of training cases? Second, how is this probability influenced by the sizes of the population and training set?  In the second half (\S\ref{s:mo}), we establish relations between lexicase selection and multi-objective optimization. Our aim is to define precisely how parents selected by lexicase variants are positioned in semantic space.}

\subsection{Expected Probabilities of Selection}\label{s:prob}
The probability of $n$ being selected by lexicase selection is the probability that a case $n$ passes is selected and: 1) $n$ is the only individual that passes the case; or 2) no more cases remain and $n$ is selected among the set of individuals that pass the selected case; or 3) $n$ is selected via the selection of another case that $n$ passes (repeating the process). 

Formally, let $P_{lex}(n | \mathcal{N}, \mathcal{T})$ be the probability of $n \in \mathcal{N}$ being selected by lexicase selection. Let $\mathcal{K}_n(\mathcal{T},\mathcal{N}) = \{k_i\}_{i=1}^K \subseteq \mathcal{T}$ be the training cases from $\mathcal{T}$ for which individual $n$ is elite among $\mathcal{N}$. We will use $\mathcal{K}_n$ for brevity. Then the probability of selection under lexicase can be represented as a piece-wise recursive function:

%{\scriptsize
%\begin{align}\label{eq:prob}
%P_{lex}(n | \mathcal{N}, \mathcal{T}) &= \\
% &\left\{\nonumber 
%     \begin{array}{lcr}
%       1 & : & |\mathcal{T}| >0, |\mathcal{N}| = 1; \\
%       1/|\mathcal{N}| & : &|\mathcal{T}| = 0; \\ 
%       \frac{1}{|\mathcal{T}|}\sum_{k_s \in K_n}{P_{lex} \left( n | \mathcal{N}(m|k_s \in K_m), \mathcal{T}(t|t \neq k_s) \right)} & : & \text{otherwise}
%     \end{array}
%   \right. 
%\end{align}
%}

\begin{equation}\label{eq:prob}
P_{lex}(n | \mathcal{N}, \mathcal{T}) =
\begin{cases}
1 & \text{if } |\mathcal{N}| = 1;\\
1/|\mathcal{N}| & \text{if } |\mathcal{T}| = 0;\\
\frac{1}{|\mathcal{T}|}\sum_{k_s \in \mathcal{K}_n}{P_{lex} \left( n | \mathcal{N}(m|k_s \in \mathcal{K}_m), \mathcal{T} \setminus \{k_s\} \right)} & \text{otherwise}
\end{cases}
\end{equation}

The first two elements of $P_{lex}$ follow from the lexicase algorithm: if there is one individual in $\mathcal{N}$, then it is selected; otherwise if there no more cases in in $\mathcal{T}$, then $n$ has a probability of selection split among the individuals in $\mathcal{N}$, i.e.,  $1/|\mathcal{N}|$. If neither of these conditions are met, the remaining probability of selection is $1/|\mathcal{T}|$ times the summation of $P_{lex}$ over $n$'s elite cases. Each case in $\mathcal{K}_n$ has a probability of $1/|\mathcal{T}|$ of being selected. For each potential selection $k_s$, the probability of $n$ being selected as a result of this case being chosen is dependent on the number of individuals that are also elite on this case, represented by $\mathcal{N}(m|k_s \in \mathcal{K}_m)$, and the cases that are left to be traversed, represented by $\mathcal{T} \setminus \{k_s\}$. 

Eqn.~\ref{eq:prob} also describes the probability of selection under $\epsilon$-lexicase selection, with the condition that {\it elitism} on a case is defined as being within $\epsilon$ of the best error on that case, where the best error is defined among the whole population (static) or among the current selection pool (semi-dynamic and dynamic) and $\epsilon$ is defined according to Eqn.~\ref{eq:ep} or Eqn.~\ref{eq:epd}. 

%\paragraph{Intuitions according to edge cases}
According to Eqn.~\ref{eq:prob}, when fitness values across the population are unique, selection probability is $P_{lex}(n) = \frac{1}{|\mathcal{T}|} \sum_{k_s \in \mathcal{K}_n} 1 = \frac{|\mathcal{K}_n|}{|\mathcal{T}|}$, since filtering the population according to any case for which $n$ is elite will result in $n$ being selected. Conversely, if the population semantics are completely homogeneous such that every individual is elite on every case, the selection will be uniformly random, giving the selection probability $P_{lex}(n) = \frac{1}{N}$. This property of uniformity in selection for identical performance holds true each time a case is considered; a case only impacts selection if there is differential performance on it in the selection pool. The same conclusion can be gleaned from Algorithm~\ref{alg:lex}: any case that every individual passes provides no selective pressure because the selection pool does not change when that case is considered.

Although it is tempting to pair Eqn.~\ref{eq:prob} with roulette wheel selection as an alternative to lexicase selection, an analysis of its complexity discourages such use. Eqn.~\ref{eq:prob} has a worst-case complexity of $O(T^N)$, which is exhibited when all individuals are elite on $\mathcal{T}$. 
%Lexicase selection can be seen as a way to sample the expected probabilities of each individual by  on random orders of cases in $\mathcal{T}$, considering one at a time rather than branching to consider other combinations of cases that could result in selection for each individual in question. 

\subsubsection{Effect of Population and Training Set Size}

Previous studies have suggested that the performance of lexicase selection is sensitive to the number of training cases~\citep{liskowski_comparison_2015}. In this section we develop the relation of population size and number of training cases to the performance of lexicase selection as a search driver. In part, this behavior is inherent to the design of the algorithm. However, this behavior is also linked to the fidelity with which lexicase selection samples the expected probabilities of selection for each individual in the population. 

The effectiveness of lexicase selection is expected to suffer when there are few training cases. When $T$ is small, there are very few ways in which individuals can be selected. In an extreme case, if $T = 2$, an individual must be elite on one of these two cases to be selected. In fact, in this case individuals with at most 2 different error vectors will be selected. For continuous errors in which few individuals are elite, this means that very few individuals are likely to produce all of the children for the subsequent generation, leading to hyperselection~\citep{helmuth_impact_2016} and diversity loss. On the other hand, if many individuals solve both cases, selection becomes increasingly random. 
%Because $\epsilon$-lexicase modulates the threshold for passing a case, it is likely to be more robust to smaller numbers of cases than lexicase selection since the case depth per selection event is likely to be higher.

The population size is tied to selection behavior because it determines the number of selection events (\texttt{ns} in Algorithms 2.1-3.3). In our implementation, \texttt{ns} = $N$, whereas in other implementations, $N \leq \texttt{ns} \leq 2N$. This implies that the value of $N$ determines the fidelity with which $P_{lex}$ is approximated via the sampling of the population by parent selection. Smaller populations will therefore produce poorer approximations of $P_{lex}$. Of course, this problem is not unique to lexicase selection; tournament selection also samples from an expected distribution and is affected by the number of tournaments~\citep{xie_another_2007}.  

Both $N$ and $T$ affect how well the expected probabilities of selection derived from Eqn.~\ref{eq:prob} are approximated by lexicase selection. Consider the probability of a case being in at least one selection event in a generation, which is one minus the probability of it not appearing, yielding
\[Pr= 1 - \prod_{i=1}^N{\frac{(T-1)!}{T!(T-1-d_i)!}} \]
Here, the case depth $d_i$ is the number of cases used to select a parent for selection event $i$. Because the case depth varies from selection to selection based on population semantics, this case probability is difficult to analyze. However, it can be simplified to consider the scenario in which a case appears {\it first} in selection. In fact, Eqn.~\ref{eq:prob} implies that a case in $\mathcal{K}_n$ influences the probability of selection of $n$ most heavily when it occurs first in a selection event. There are two reasons: first, the case has the potential to filter $N-1$ individuals, which is the strongest selection pressure it can apply. Second, a case's effect size is highest when selected first because it is not conditioned on the probability of selection of any other cases. Each subsequent case selection has a reduced effect on $P_{lex}$ of $\prod_{i=1}^d{\frac{1}{T-i}}$, where $d$ is the case depth. These observations also highlight the importance of the relative sizes of $N$ and $T$ because they affect the probability that a case will be observed at the top of a selection event in a given generation, which affects how closely Eqn.~\ref{eq:prob} is approximated. Let $P_{\text{first}}$ be the probability that a case will come first in a selection event {\it at least once} in a generation. Then
\begin{equation}\label{eq:prob_case}
P_{\text{first}} = 1 - \left(\frac{(T-1)}{T}\right)^N
\end{equation}
assuming $N$ selection events. This function is plotted for various values of $N$ and $T$ in Figure~\ref{fig:prob_first}, and illustrates that the probability of a case appearing first in selection drops as $T$ grows and as $N$ shrinks. For example, $P_{\text{first}} \approx 0.5$ when $N=1000$ and $T=1433$. We therefore expect the observed probabilities of selection for $n \in \mathcal{N}$ to differ from $P_{lex}(n)$ when $T >> N$, due to insufficient sampling of the cases. In the case of $N >> T$, we expect most cases to appear first and therefore the probability predictions made by Eqn.~\ref{eq:prob} to be more accurate to the actual selections. 

\begin{figure}
\centering
  \includegraphics[width = 0.6\textwidth]{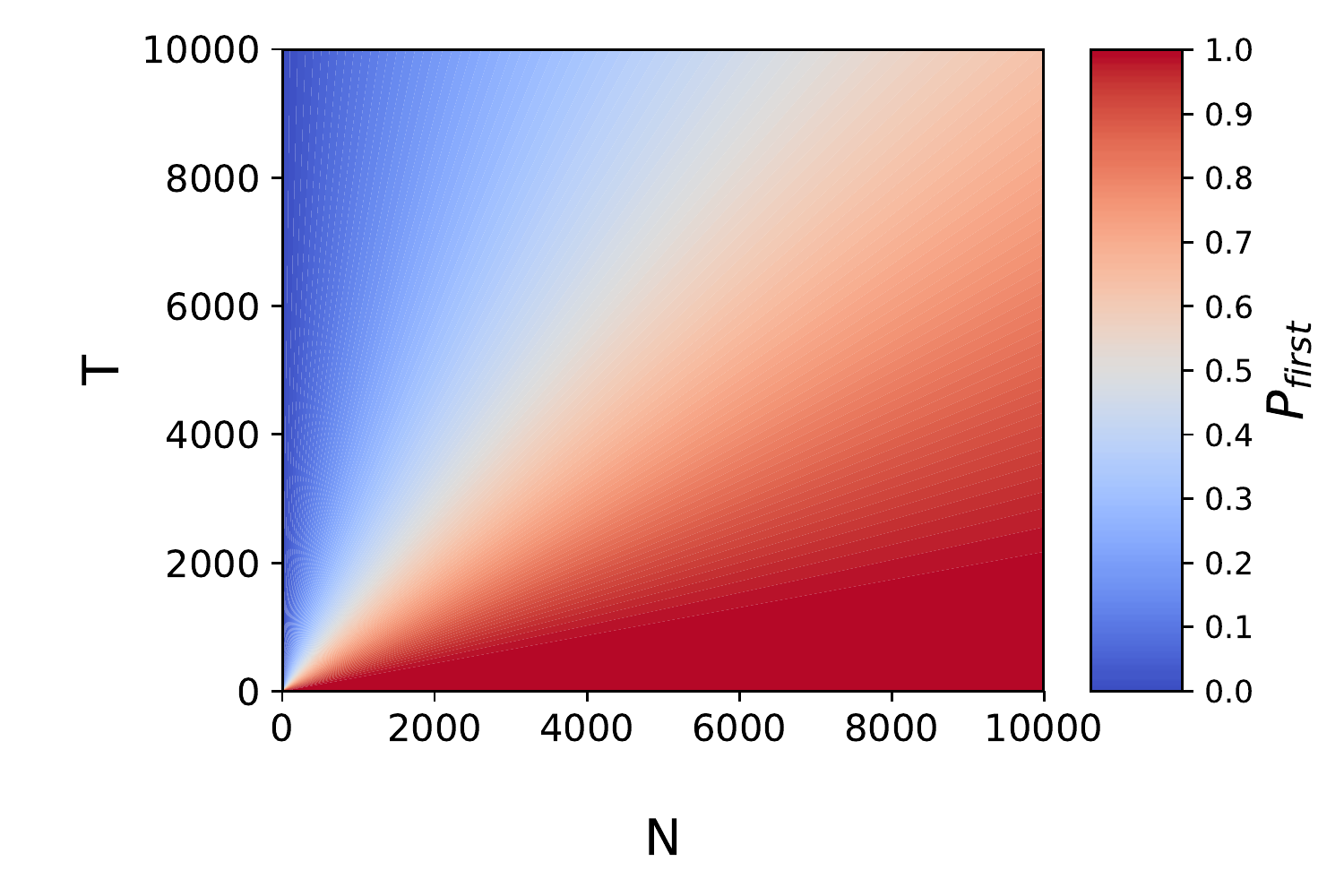}
  \caption{The probability of a case occuring first in a selection event given $T$ training cases and $N$ selections.}\label{fig:prob_first}
\end{figure}

\subsubsection{Probabilities under tournament selection} We compare the probability of selection under lexicase selection to that using tournament selection with an identical population and fitness structure. To do so we must first formulate the probability of selection for an individual undergoing tournament selection with size $r$ tournaments. The fitness ranks of $\mathcal{N}$ can be calculated, for example using MAE as fitness, with lower rank indicating better fitness. Let $S_i$ be the individuals in $\mathcal{N}$ with a fitness rank of $i$, and let $Q$ be the number of unique fitness ranks. \cite{xie_another_2007} showed that the probability of selecting an individual with rank $j$ in a single tournament is
\begin{equation}\label{eq:prob_t}
P_t = \frac{1}{|S_j|}\left( \left(\frac{\sum_{i=j}^Q{|S_i|}}{N}\right)^r - \left(\frac{\sum_{i=j+1}^Q{|S_i|}}{N}\right)^r \right)
\end{equation}

In Table~\ref{tbl:ex}, the selection probabilities for the example population are shown according to lexicase selection (Eqn.~\ref{eq:prob}) and tournament selection (Eqn.~\ref{eq:prob_t}). Note that the tournament probabilities are proportional to the aggregate fitness, whereas lexicase probabilities reflect more subtle but intuitive performance differences as discussed by~\cite{spector_assessment_2012}. In \S\ref{s:ex} we present a more detailed population example with continuous errors and compare probabilities of selection using lexicase, $\epsilon$-lexicase and tournament selection. 

\subsection{Multi-objective Interpretation of Lexicase Selection}\label{s:mo}

Objectives and training cases are fundamentally different entities: objectives define the goals of the task being learned, whereas cases are the units by which progress towards those objectives is measured. By this criteria, lexicase selection and multi-objective optimization have historically been differentiated~\citep{helmuth_general_2015}, although there is clearly a ``multi-objective'' interpretation of the behavior of lexicase selection with respect to the training cases. Let us assume for the remainder of this section that individual fitness cases are objectives to solve. The symbolic regression task then becomes a high-dimensional, many-objective optimization problem. At this scale, the most popular multi-objective methods (e.g. NSGA-II and SPEA-2) tend to perform poorly, a behavior that has been explained in literature~\citep{wagner_pareto-_2007, farina_optimal_2002}. \cite{farina_optimal_2002} point out two short-comings of these multi-objective methods when many objectives are considered: \begin{quote}
the Pareto definition of optimality in a multi-criteria decision making problem can be unsatisfactory due to essentially two reasons: the number of improved or equal objective values is not taken into account, the (normalized) size of improvements is not taken into account.
\end{quote}

As we describe in \S\ref{s:prob}, lexicase selection takes into account the number of improved or equal objectives (i.e. cases) by increasing the probability of selection for individuals who solve more cases (consider the summation in the third part of Eqn.~\ref{eq:prob}). The increase per case is proportional to the difficulty of that case, as defined by the selection pool's performance. Regarding Farina and Amato's second point, the {\it size} of the improvements are taken into account by $\epsilon$-lexicase selection. They are taken into account by the automated thresholding performed by $\epsilon$ which rewards individuals for being within an acceptable range of the best performance on the case. We develop the relationship between lexicase selection and Pareto optimization in the remainder of this section. 

It has been noted that lexicase selection guarantees the return of individuals that are on the Pareto front with respect to the fitness cases~\citep{la_cava_epsilon-lexicase_2016}. However, this is a necessary but not sufficient condition for selection. As we show below, lexicase selection only selects those individuals in the ``corners" or boundaries of the Pareto front, meaning they are on the front {\it and} elite, globally, with respect to at least one fitness case. Below, we define these Pareto relations with respect to the training cases. 

\begin{defn}\label{def:dom}
$n_1$ {\it dominates} $n_2$, i.e., ${n_1} \prec {n_2}$, if $e_j(n_1) \leq e_j(n_2) \;
\forall j  \in \{1,\dots,T\}$ and $\exists j \in \{1,\dots,T\}$ for which $e_j(n_1) < e_j(n_2)$. \bigskip
\end{defn}

\begin{defn}\label{def:pset}
The {\it Pareto set} of $\mathcal{N}$ is the subset of $\mathcal{N}$ that is non-dominated with respect to $\mathcal{N}$; i.e., $n \in \mathcal{N}$ is in the Pareto set if $m \nprec n \; \forall \; m \in \mathcal{N}$. 
\end{defn}

\begin{defn}\label{def:boundary}
$n \in \mathcal{N}$ is a {\it Pareto set boundary} if $n \in$ Pareto set of $\mathcal{N}$ and $\exists j \in \{1,\dots,T\}$ for which $e_j(n) \leq e_j(m) \; \forall \; m \in \mathcal{N}$. 
\end{defn}

With these definitions in mind, we show that individuals selected by lexicase are Pareto set boundaries. 

\begin{lex}\label{thm:lex}
If individuals from a population $\mathcal{N}$ are selected by lexicase selection, those individuals are Pareto set boundaries of $\mathcal{N}$ with respect to $\mathcal{T}$. 
\end{lex}

%\noindent \textit{Proof:} %First, we prove by contradiction that any individual selected by lexicase selection is in the Pareto set. Second, we prove that any such individual is a Pareto set boundary. 
\begin{proof}
Let $n_1 \in \mathcal{N}$ be any individual and let $n_2$ $\in \mathcal{N}$ be an individual selected from $\mathcal{N}$ by lexicase selection. Suppose $n_1 \prec n_2$. Then $e_j(n_1) \leq e_j(n_2) \;
\forall j  \in \{1,\dots,T\}$ and $\exists j \in \{1,\dots,T\}$ for which $e_j(n_1) < e_j(n_2)$. Therefore $n_1$ remains in the selection pool for every case that $n_2$ does, yet $\exists t \in \mathcal{T}$ for which $n_2$ is removed from every selection event due to $n_1$. Hence, $n_2$ cannot be selected by lexicase selection and the supposition is false. Therefore $n_2$ must be in the Pareto set of $\mathcal{N}$. 

Next, Algorithm~\ref{alg:lex} shows that $n_2$ must be elite on at least one test case; therefore $\exists j \in \{1,\dots,T\}$ for which $e_j(n_2) \leq e_j(m) \; \forall \; m \in \mathcal{N}$. Therefore, since $n_2$ is in the Pareto set of $\mathcal{N}$, according to Definition~\ref{def:boundary}, $n_2$ is a Pareto set boundary of $\mathcal{N}$.  
\end{proof}
\bigskip

\paragraph{Extension to $\epsilon$-lexicase selection}
We can extend our multi-objective analysis to $\epsilon$-lexicase selection for conditions in which $\epsilon$ is pre-defined for each fitness case  (Eqn.~\ref{eq:ep}), which is true for static and semi-dynamic $\epsilon$-lexicase selection. However when $\epsilon$ is recalculated for each selection pool, the theorem is not as easily extended due to the need to account for the dependency of $\epsilon$ on the current selection pool. We first define $\epsilon$ elitism in terms of a relaxed dominance relation and a relaxed Pareto set. We define the dominance relation with respect to $\epsilon$ as follows:

\begin{defn}\label{def:edom}
$n_1$ {\it $\epsilon$-dominates} $n_2$, i.e., ${n_1} \prec_{\epsilon} {n_2}$, if $e_j(n_1) + \epsilon_j \leq e_j(n_2)  \;
\forall j  \in \{1,\dots,T\}$ and $\exists j \in \{1,\dots,T\}$ for which $e_j(n_1) + \epsilon_j < e_j(n_2) $, where $\epsilon_j>0$ is defined according to Eqn.~\ref{eq:ep}.
\end{defn}

This definition of $\epsilon$-dominance differs from a previous $\epsilon$-dominance definition used by~\cite{laumanns_archiving_2002} (cf. Eqn. (6)) in which ${n_1} \prec_{\epsilon} {n_2}$ if \[ e_j(n_1) + \epsilon_j \leq e_j(n_2) \; \forall \; j  \in \{1,\dots,T\}\] 
Definition~\ref{def:edom} is more strict, requiring $e_j(n_1) + \epsilon_j < e_j(n_2)$ for at least one $j$ in analagous fashion to Definition~\ref{def:dom}.  In order to extend Theorem~\ref{thm:lex}, this definition must be more strict since a useful $\epsilon$-dominance relation needs to capture the ability of an individual to preclude the selection of another individual under $\epsilon$-lexicase selection. 
%As a result, the non-$\epsilon$-dominated set, defined as the $\epsilon$-Pareto set below, is expected to be as large or larger than the $\epsilon$-approximate Pareto set used in~\citep{laumanns_archiving_2002}. 
 
\begin{defn}\label{def:epset}
The {\it $\epsilon$-Pareto set} of $\mathcal{N}$ is the subset of $\mathcal{N}$ that is non-$\epsilon$-dominated with respect to $\mathcal{N}$; i.e., $n \in \mathcal{N}$ is in the $\epsilon$-Pareto set if $m \nprec_{\epsilon} n \; \forall \; m \in \mathcal{N}$. 
\end{defn}

\begin{defn}\label{def:eboundary}
$n \in \mathcal{N}$ is an {\it $\epsilon$-Pareto set boundary} if $n$ is in the $\epsilon$-Pareto set of $\mathcal{N}$ and $\exists j \in \{1,\dots,T\}$ for which $e_j(n_1) \leq e_j(m) + \epsilon_j \; \forall \; m \in \mathcal{N}$, where $\epsilon_j$ is defined according to Eqn~\ref{eq:ep}. \bigskip
\end{defn}

\begin{lex}\label{thm:eplex}
If $\epsilon$ is defined according to Eqn.~\ref{eq:ep}, and if individuals are selected from a population $\mathcal{N}$ by $\epsilon$-lexicase selection, then those individuals are $\epsilon$-Pareto set boundaries of $\mathcal{N}$.  
\end{lex}

\begin{proof}
Let $n_1 \in \mathcal{N}$ be any individual and let $n_2$ $\in \mathcal{N}$ be an individual selected from $\mathcal{N}$ by static or semi-dynamic $\epsilon$-lexicase selection. Suppose $n_1 \prec_{\epsilon} n_2$. Therefore $n_1$ remains in the selection pool for every case that $n_2$ does, yet $\exists t \in \mathcal{T}$ for which $n_2$ is removed from every selection event due to $n_1$. Hence, $n_2$ cannot be selected by lexicase selection and the supposition $n_1 \prec_{\epsilon} n_2$ is false.  Therefore $n_1$ and $n_2$ must be in the $\epsilon$-Pareto set of $\mathcal{N}$ to be selected. 

Next, by definition of Algorithm~\ref{alg:ep-lex-s} or 3.2, $n_2$ must be within $\epsilon_j$ of elite on at least one test case; i.e. $\exists j \in \{1,\dots,T\}$ for which $e_j(n_2) \leq e_j(m) + \epsilon_j \; \forall \; m \in \mathcal{N}$. Therefore, since $n_2$ is in the $\epsilon$-Pareto set of $\mathcal{N}$, according to Definition~\ref{def:eboundary}, $n_2$ must be a $\epsilon$-Pareto set boundary of $\mathcal{N}$.  
\end{proof}
\bigskip

To illustrate how lexicase selection only selects Pareto set boundaries, we plot an example selection from a population evaluated on two test cases in the left plot of Figure~\ref{fig:lex_pareto}. Each point in the plot represents an individual, and the squares are the Pareto set. Under a lexicase selection event with case sequence $\{t_1, t_2\}$, individuals would first be filtered to the two left-most individuals that are elite on $e_1$, and then to the individual among those two that is best on $e_2$, i.e. the selected square individual. Note that the selected individual is a Pareto set boundary. The individual on the other end of the Pareto set shown as a black square would be selected using the opposite order of cases.  

Consider the analogous case for semi-dynamic $\epsilon$-lexicase selection illustrated in the right plot of Figure~\ref{fig:lex_pareto}. In this case the squares are the $\epsilon$-Pareto set. Under a semi-dynamic $\epsilon$-lexicase selection event with case order $\{t_1, t_2\}$, the population would first be filtered to the four left-most individuals that are within $\epsilon_1$ of the elite error on case $t_1$, and then the indicated square would be selected by being the only individual within $\epsilon_2$ of the elite error on $t_2$ among the current pool. Note that although the selected individual is an $\epsilon$-Pareto set boundary by Definition~\ref{def:eboundary}, it is {\it not} a boundary of the Pareto set. Theorem~\ref{thm:eplex} only guarantees that the selected individual is within $\epsilon$ of the best error for at least one case, which in this scenario is $t_1$. Thus Figure~\ref{fig:lex_pareto} illustrates an important aspect of introducing $\epsilon$: it reduces the selectivity of each case, ultimately resulting in the selection of individuals that are not as extremely positioned in objective space. \edit{This parallels the behavior of $\epsilon$-dominance methods proposed by~\cite{laumanns_archiving_2002} that can lose extreme points. In light of this potential limitation, \cite{hernandez-diaz_pareto-adaptive_2007} proposed an adaptive $\epsilon$-dominance method to preserve such boundary points.}

\begin{figure}[tb]
\begin{minipage}{0.49\textwidth}
\centering
  \includegraphics[width = \textwidth]{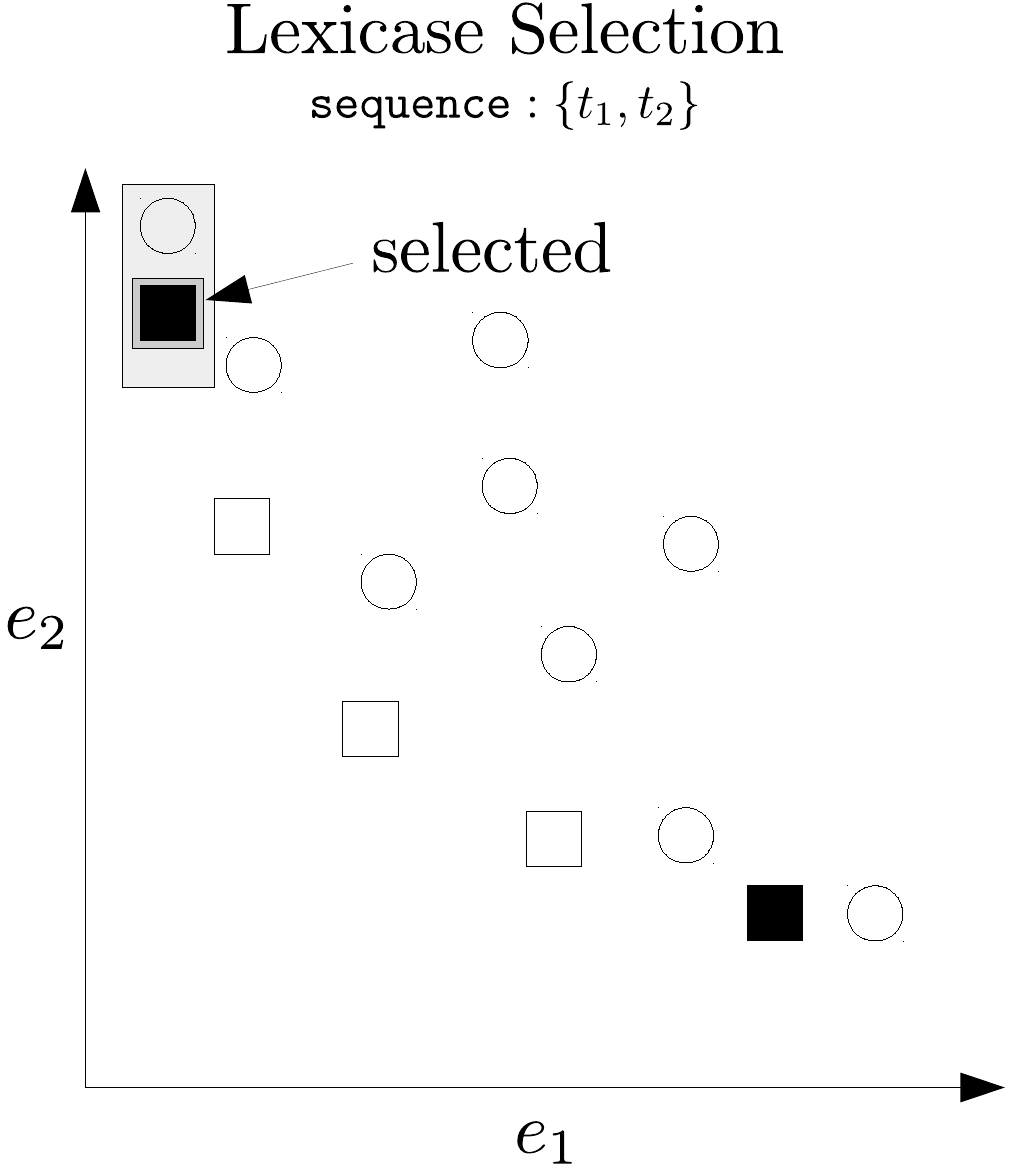}
  %\caption{An illustration of the performance of lexicase selection in a scenario involving two cases. Each point represents and individual in the population. The squares are individuals in the Pareto set. A selection event for case sequence $\{t_1,t_2\}$ is shown by the gray rectangles. The black points are individuals that could be selected by any case ordering.}\label{fig:lex_pareto}
%\end{figure}
\end{minipage}
\hspace{0.02\textwidth}
\begin{minipage}{0.49\textwidth}
%\begin{figure}
\centering
  \includegraphics[width = \textwidth]{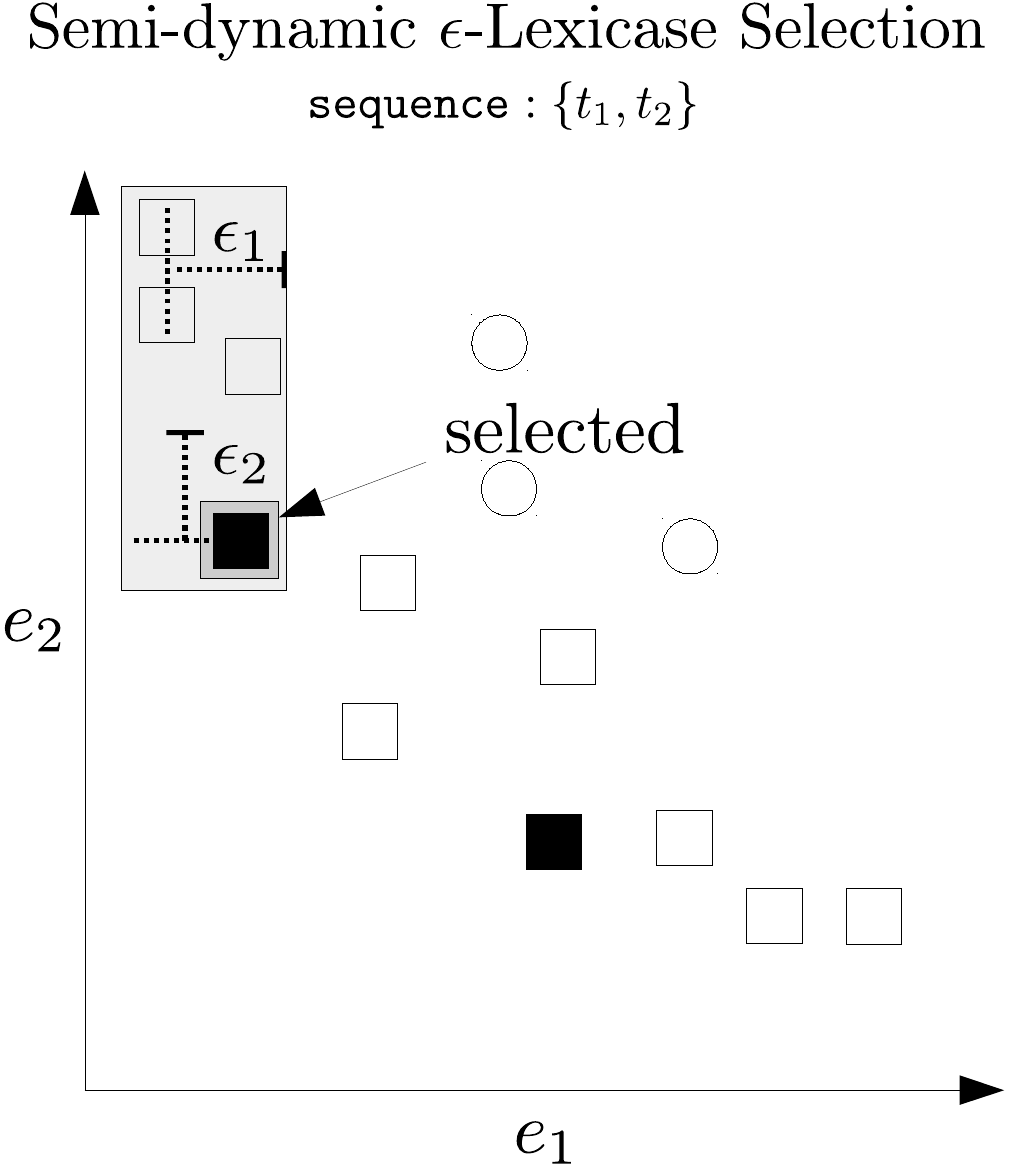}
\end{minipage}
\caption{An illustration of the performance of lexicase selection (left) and semi-dynamic $\epsilon$-lexicase selection (right) in a scenario involving two cases. Each point represents and individual in the population. The squares in the left figure are individuals in the Pareto set, and the squares on the right are individuals in the $\epsilon$-Pareto set. A selection event for case sequence $\{t_1,t_2\}$ is shown by the gray rectangles. The black points are individuals that could be selected by any case ordering.}\label{fig:lex_pareto}
\end{figure}
 
Regarding the position of solutions in this space, it's worth noting the significance of boundary solutions (and near boundary solutions) in the context of multi-objective optimization. \edit{The performance of algorithms in many-objective optimization is assessed by convergence, uniformity, and spread~\citep{li_spread_2009}, the last of which deals directly with the extent of boundary solutions. Indicator-based methods such as IBEA and SMS-EMOA use a measure of the hypervolume in objective space to evaluate algorithm performance~\citep{wagner_pareto-_2007}, where the hypervolume is a measure of how well the objective space is covered by the current set of solutions. Boundary solutions have been shown empirically to contribute significantly to hypervolume measures~\citep{deb_evaluating_2005}. The boundary solutions have an infinite score according to NSGA-II's crowding measure~\citep{schoenauer_fast_2000}, with higher being better, meaning they are the first non-dominated solutions to be preserved by selection when the population size is reduced. However, \cite{auger_theory_2009} showed mathematically that the position of solutions near the boundary is less important than the angle they form with other solutions when evaluating the hypervolume. \cite{auger_theory_2009} concluded that ``Extreme points are not generally preferred as claimed in~\citep{deb_evaluating_2005}, since the density of points does not depend on the position on the front but only on the gradient at the respective point".

Multi- and many-objective literature is therefore divided on how these boundary solutions drive search when the goal of the algorithm is to approximate the optimal Pareto front~\citep{wagner_pareto-_2007}.} The goal of GP, in contrast, is to preserve points in the search space that, when combined and varied, yield a single best solution. So while the descriptions above lend insight to the function of lexicase and $\epsilon$-lexicase selection, the different goals of search and the high dimensionality of training cases must be remembered when \edit{drawing parallels between these approaches.} 

As a last note, when considered as objectives, the worst-case complexity of lexicase selection matches that of NSGA-II: $O(TN^2)$. Interestingly, the worst case complexity of the crowding distance assignment portion of NSGA-II, $O(TN\log(N))$, occurs when all individuals are non-dominated, which is expected in high dimensions~\citep{farina_optimal_2002, wagner_pareto-_2007}. Under lexicase selection, a non-dominated population {\it that is semantically unique} will have a worst-case complexity of $O(N^2)$.

%\paragraph{Relation to Evolutionary Multi-objective Optimization}
%
%There is an interesting relationship to be made regarding the complexity of lexicase selection in comparison to NSGA-II. Consdering training cases as objectives, we see that lexicase selection and NSGA-II have the same worst-case complexity of $O(T^2N)$. However the algorithms differ with respect to the conditions under which worst case complexities arise. 
%
%NSGA-II has three main parts: sorting ($O(TN^2)$), crowding distance assignment $O(TN \log(N))$, and sorting with crowding comparison $O(2Nlog(2N))$, giving an overall complexity of $O(TN^2$). The sorting complexity is $O(TN^2)$ to identify a single front; the crowding distance assignment complexity varies, however. Its worst case complexity arises when all individuals are non-dominated, which is expected in high dimensions~\citep{farina_optimal_2002, wagner_pareto-_2007}. 
%
%Interestingly, {\it if the population is unique}, this is the minimum complexity scenario for lexicase selection: in other words, if the semantics of the population are unique and all are non-dominated in $\mathcal{T}$, only case depth 1 selections will occur, giving a runtime complexity of $O(N^2)$. Of course, the population could also be non-dominated by being semantically identical, in which case the complexity would be $O(TN^2)$. The important thing to note is that lexicase selection, which rewards individuals for being diverse, also runs more quickly when the population is more diverse. 

\edit{\section{Experimental Analysis}}
\edit{We begin our experimental analysis of lexicase selection by considering an illustrative example in \S\ref{s:ex}. We then test several parent selection strategies on a set of regression benchmarks in \S\ref{s:exp}. Finally, we quantify wall-clock runtimes for various selection methods as the population size is increased. \footnote{\edit{Code for these experiments: \url{http://github.com/lacava/epsilon_lexicase}}}}
\subsection{Illustrative Example}\label{s:ex}
Here we apply the concepts from \S\ref{s:prob} to consider the probabilities of selection under different methods on an example population. The goal of this section is to interweave the analyses of \S\ref{s:prob} and \S\ref{s:mo} to give an intuitive explanation of the differences between lexicase selection and the $\epsilon$-lexicase selection variants. 

An example population is presented in Table~\ref{tbl:ex2} featuring floating point errors, in contrast to Table~\ref{tbl:ex}. In this case, the population semantics are completely unique, although they result in the same mean error across the training cases, as shown in the ``Mean" column. As a result, tournament selection picks uniformly from among these individuals, resulting in equivalent probabilities of selection. As mentioned in \S\ref{s:prob}, with unique populations, lexicase selection is proportional to the number of cases for which an individual is elite. This leads lexicase selection to pick from among the four individuals that are elite on cases, i.e. $n_1$, $n_4$, $n_5$, and $n_9$, with respective probabilities 0.2, 0.2, 0.2, and 0.4, according to Eqn.~\ref{eq:prob}. Note these four individuals are Pareto set boundaries. 

\begin{table}[htb]
\centering
\scriptsize
\caption{Example population with training case performances and selection probabilities according to the different algorithms.}\label{tbl:ex2}
\rowcolors{1}{white}{LightCyan}
\begin{tabularx}{0.75\textwidth}{X|rrrrr|r|R{2em}R{2em}R{2em}R{2em}R{2em}}\toprule
$\mathcal{N}$ & \multicolumn{5}{c}{Cases} && \multicolumn{5}{c}{Probability of Selection} \\
& $e_1$ & $e_2$ & $e_3$ & $e_4$ & $e_5$ & Mean &	tourn	&	lex	&	$\epsilon$ lex static	&	$\epsilon$ lex semi	&	$\epsilon$ lex dyn\\ \midrule
$n_1$	&0.0	&	1.1	&	2.2	&	3.0	&	5.0 & 2.26	&	0.111	&	0.200	&	0.000	&	0.067	&	0.033\\ 
$n_2$	&0.1	&	1.2	&	2.0	&	2.0	&	6.0 & 2.26	&	0.111	&	0.000	&	0.150	&	0.117	&	0.200\\ 
$n_3$	&0.2	&	1.0	&	2.1	&	1.0	&	7.0 & 2.26	&	0.111	&	0.000	&	0.150	&	0.117	&	0.117\\ 
$n_4$	&1.0	&	2.1	&	0.2	&	0.0	&	8.0 & 2.26	&	0.111	&	0.200	&	0.300	&	0.200	&	0.167\\ 
$n_5$	&1.1	&	2.2	&	0.0	&	4.0	&	4.0 & 2.26	&	0.111	&	0.200	&	0.000	&	0.050	&	0.050\\ 
$n_6$	&1.2	&	2.0	&	0.1	&	5.0	&	3.0 & 2.26	&	0.111	&	0.000	&	0.000	&	0.050	&	0.033\\ 
$n_7$	&2.0	&	0.1	&	1.2	&	6.0	&	2.0 & 2.26	&	0.111	&	0.000	&	0.133	&	0.133	&	0.133\\ 
$n_8$	&2.1	&	0.2	&	1.0	&	7.0	&	1.0 & 2.26	&	0.111	&	0.000	&	0.133	&	0.133	&	0.217\\ 
$n_9$	&2.2	&	0.0	&	1.1	&	8.0	&	0.0 & 2.26	&	0.111	&	0.400	&	0.133	&	0.133	&	0.050\\  \midrule
$\epsilon$	&	0.9	& 0.9	&	0.9	&	2.0	& 2.0	&&&&&&\\ \bottomrule
\end{tabularx}
\end{table}

Due to its strict definition of elitism, lexicase selection does not account for the fact that other individuals are very close to being elite on these cases as well; for example $n_2$ and $n_3$ are close to the elite error on case $t_1$. The $\epsilon$-lexicase variants address this as noted by the smoother distribution of selection probabilities among this population. We focus first on static $\epsilon$-lexicase selection. Applying the $\epsilon$ threshold to the errors yields the following discrete fitnesses:
\begin{center}
\begin{tabular}{lrrrrr}
& $e_1$ & $e_2$ & $e_3$ & $e_4$ & $e_5$ \\
$n_1$	&0	&	1	&	1	&	1	&	1\\ 
$n_2$	&0	&	1	&	1	&	0	&	1\\ 
$n_3$	&0	&	1	&	1	&	0	&	1\\ 
$n_4$	&1	&	1	&	0	&	0	&	1\\ 
$n_5$	&1	&	1	&	0	&	1	&	1\\ 
$n_6$	&1	&	1	&	0	&	1	&	1\\ 
$n_7$	&1	&	0	&	1	&	1	&	0\\ 
$n_8$	&1	&	0	&	1	&	1	&	0\\ 
$n_9$	&1	&	0	&	1	&	1	&	0\\ 
\end{tabular}
\end{center}

The selection probabilities for static $\epsilon$-lexicase selection are equivalent to the selection probabilities of lexicase selection on this converted error matrix. Note that $n_1$ and $n_5$ have selection probabilities of zero because they are dominated in the {\it converted} error space. Despite elitism on case $t_1$, $n_1$ is not selected since $n_2$ and $n_3$ are $\epsilon$-elite on this case {\it in addition to} $t_4$. The same effect makes $n_5$ un-selectable due to $n_4$. Consider $n_4$, which has a higher probability of selection under static $\epsilon$-lexicase selection than lexicase selection. This is due to $n_4$ being $\epsilon$-elite on a unique combination of cases: $t_3$ and $t_4$. Lastly, $n_9$ is selected in equal proportions to $n_7$ and $n_8$ because all three are within $\epsilon$ of the elite error on the same cases. 

Semi-dynamic $\epsilon$-lexicase selection allows for all nine individuals to be selected with varying proportions that are similar to those derived for static $\epsilon$-lexicase selection. Selection probabilities for $n_1$ illustrate the differences in the static and semi-dynamic variants: $n_1$ has a chance for selection in the semi-dynamic case because when $t_1$ is selected as the first case, $n_1$ is within $\epsilon$ of the best case errors {\it among the pool}, i.e. $\{n_1$, $n_2$, $n_3\}$, for any subsequent order of cases. The probability of selection for $n_5$ and $n_6$ follow the same pattern.

Dynamic $\epsilon$-lexicase selection produces the most differentiated selection pressure for this example. Consider individual $n_8$ which is the most likely to be selected for this example. It is selected more often than $n_7$ or $n_9$ due to the adaptations to $\epsilon$ as the selection pool is winnowed. For example, $n_8$ is selected by case sequence $\{t_2,t_1,t_3\}$, for which the selection pool takes the following form after each case: $\{n_7, n_8, n_9\}, \{n_7, n_8\}, \{n_8\}$. Conversely, under semi-dynamic $\epsilon$-lexicase selection, $n_7$ and $n_9$ would not be removed by these cases because $\epsilon$ is fixed for that variant.

\subsection{Regression Experiments} \label{s:exp}
In this section we empirically test the variants of $\epsilon$-lexicase selection introduced in \S\ref{s:eplex}. The problems studied in this section are listed in Table~\ref{tbl:regression}. We benchmark nine methods using eight different datasets. Six of the problems are available from the UCI repository~\citep{lichman_uci_2013}. The UBall5D problem is a simulated equation\footnote{UBall5D is also known as Vladislavleva-4.} which has the form \[ y = \frac{10}{5+\sum_{i=1}^5{(x_i-3)^2}}\] The Tower problem and UBall5D were chosen from the benchmark suite suggested by \cite{white_better_2012}.

\begin{table}
\begin{minipage}{0.35\textwidth}

\centering
\scriptsize
\caption{Regression problems used for method comparisons.}\label{tbl:regression}
\begin{tabularx}{\textwidth}{X r r } \toprule
Problem & Dimension & Samples \\ \midrule
Airfoil & 5	& 1503 \\
Concrete	& 	8	& 1030	\\
%Crime	&	127	&	1993	\\
ENC & 8 & 768 \\
ENH & 8 & 768 \\
Housing & 14 & 506 \\
Tower & 25 & 3135 \\
UBall5D & 5 & 6024 \\ 
Yacht	& 6	&	309	\\ \midrule
\end{tabularx}
\end{minipage}
\hspace{0.05\textwidth}
\begin{minipage}{0.6\textwidth}

%\begin{table}
\centering
\scriptsize
\caption{GP settings.}\label{tbl:symreg_settings}
\begin{tabularx}{\textwidth}{Xr} \toprule
Setting& Value \\ \midrule
Population size & 1000 \\
Crossover / mutation & 60/40\% \\
Program length limits & [3, 50] \\ 
ERC range & [-1,1] \\
Generation limit & 1000 \\
Trials & 50 \\
Terminal Set & \{$\mathbf{x}$, ERC, $+$, $-$, $*$, $/$, $\sin$, $\cos$, $\exp$, $\log$\}\\
Elitism & keep best \\
Fitness (non-lexicase methods) & MSE \\ \bottomrule
\end{tabularx}
%\end{table}
\end{minipage}
\end{table}

We compare eight different selection methods: random selection, tournament selection, lexicase selection, age-fitness pareto optimization~\citep{schmidt_age-fitness_2011}, deterministic crowding~\citep{mahfoud_niching_1995}, and the three $\epsilon$-lexicase selection methods presented in \S\ref{s:eplex}. In addition to the selection methods that are benchmarked, we include a comparison to regularized linear regression using Lasso~\citep{tibshirani_regression_1996}. These methods are described briefly below, along with their abbreviations used in the results.
\begin{itemize}
\item Random Selection (rand): selection for parents is uniform random.
\item Tournament Selection (tourn): size two tournaments are conducted for choosing parents. 
\item Lexicase Selection (lex): see Algorithm~\ref{alg:lex}. 
\item Age-fitness Pareto optimization (afp): this method introduces a new individual each generation with an age of 0. Each generation, individuals are assigned an age equal to the number of generations since their oldest ancestor entered the population. Parents are selected randomly to create $N$ children. The children and parents then compete in survival tournaments of size two, in which an individual is culled from the population if it is dominated in terms of age and fitness by its competitor. 
\item Deterministic crowding (dc): A generational form of this niching method is used in which parents are selected randomly for variation and the child competes to replace the parent with which it is most similar. Similarity is determined based on the Levenshtein distance of the parent's equation forms, using a universal symbol for coefficients. A child replaces its parent in the population only if it has a better fitness.
\item Static $\epsilon$-lexicase selection (ep-lex-s): See Algorithm~\ref{alg:ep-lex-s}.
\item Semi-dynamic $\epsilon$-lexicase selection (ep-lex-sd): See Algorithm~\ref{alg:ep-lex-sd}.
\item Dynamic $\epsilon$-lexicase selection (ep-lex-d): See Algorithm~\ref{alg:ep-lex-d}.
\item Lasso (lasso): this method incorporates a regularization penalty into least squares regression using an $\ell_1$ measure of the model coefficients and uses a tuning parameter, $\lambda$, to specify the weight of this regularization. We use a least angle regression~\citep{efron_least_2004} implementation of Lasso that automatically chooses $\lambda$ using cross validation.
\end{itemize} 

The settings for the GP system\footnote{available from \url{https://epistasislab.github.io/ellyn/}} are shown in Table~\ref{tbl:symreg_settings}. We conduct 50 trials of each method by training on a random partition of 70\% of the dataset and comparing the prediction error of the best model from each method on the other 30\% of the dataset. In addition to test error, we compare the training convergence of the GP-based methods, the semantic diversity of the populations during the run, and the number of cases used for selection for the lexicase methods. We calculate population diversity as the fraction of unique semantics in the population. To compare the number of cases used in selection for the lexicase methods, we save the median number of cases used in selection events, i.e. the case depth, each generation.

\subsubsection{Regression Results}\label{s:results}
The boxplots in Figure~\ref{fig:boxplot_reg} show the test set MSE for each method on each problem. In the final subplot, we summarize the mean rankings of the methods on each trial of each problem to give a general comparison of performance. Ranks are calculated for each trial, and then averaged over all trials and problems to give an overall ranking comparison. In general we find that the $\epsilon$-lexicase methods produce models with the best generalization performance across the tested problems. Random selection and Lasso tend to perform the worst on these problems. It is interesting to note the performance of Lasso on the Tower problem, which is better than on other datasets; ep-lex-sd and ep-lex-d are the only GP variants to significantly outperform it. For every problem, a variant of $\epsilon$-lexicase selection performs the best, and the three variants of it tend to perform similarly. In accordance with previous results~\citep{la_cava_epsilon-lexicase_2016}, lexicase selection performs worse than tournament selection for these continuous valued problems. In contrast with previous findings~\citep{schmidt_age-fitness_2011}, dc tends to outperform afp, although both methods perform better than tournament selection. 

The $\epsilon$-lexicase methods show a marked advantage in converging on a low training error in fewer generations compared to all other methods, as evidenced in Figure~\ref{fig:train}. Note Figure~\ref{fig:train} reports the normalized MSE values on the training set for the best individual in the population each generation. Again we observe very little difference between the $\epsilon$-lexicase variants. 

We analyze the statistical significance of the test MSE results in Tables~\ref{tbl:wilcox} and~\ref{tbl:hsd}. Table~\ref{tbl:wilcox} shows pair-wise Wilcoxon ranksum tests for each method in comparison to ep-lex-sd. There are significant differences in performance for all problems between ep-lex-sd and all non-$\epsilon$-lexicase methods, with the exception of the comparison to dc on the housing and tower datasets. Analysis of variance of the method rankings across all problems indicates significant differences ($p<$ 2e-16). A post-hoc statistical analysis shown in Table~\ref{tbl:hsd} indicates that this difference is due to significant differences in rankings across all problems for ep-lex-sd and ep-lex-d in pairwise comparison to all other non-$\epsilon$-lexicase methods. The three variants of $\epsilon$-lexicase do not differ significantly from each other according to this test.

Figure~\ref{fig:diversity} shows the semantic diversity of the populations for each generation using different selection methods. $\epsilon$-lexicase variants, dc, and lexicase selection all produce the highest population diversity, as expected due to their diversity maintenance design. Interestingly, they all produce more diverse semantics than random selection, suggesting that the preservation of {\it useful} diversity is an important feature of the observed performance improvements. Surprisingly, afp is found to produce low semantic diversity, despite its incorporation of age and random restarts each generation. Given that afp has no explicit semantic diversity mechanism, it's possible that age is not an adequate surrogate for behavioral diversity on these problems. 
  
One of the motivations for introducing an $\epsilon$ threshold into lexicase selection is to allow selection to make use of more cases in continuous domains when selecting parents. Figure~\ref{fig:case_depth} demonstrates that $\epsilon$-lexicase methods achieve this goal. As we noted at the beginning of \S\ref{s:eplex}, lexicase selection likely only uses one case per selection event in continuous domains, leading to poor performance. We observe this phenonemon in the median case depth measurements. Among the $\epsilon$-lexicase variants, ep-lex-sd uses the most cases in general, followed by ep-lex-s and ep-lex-d. Intuitively this result makes sense: $\epsilon$ is likely to be largest when computed across the population, and because ep-lex-sd uses the global $\epsilon$ (Eqn.~\ref{eq:ep}) and a local error threshold, it is likely to keep the most individuals at each case filtering. These results also suggest that $\epsilon$ shrinks substantially when calculated among the pool after each case (Eqn.~\ref{eq:epd}) in ep-lex-d.

\begin{figure}
\centering
  \includegraphics[width=\textwidth]{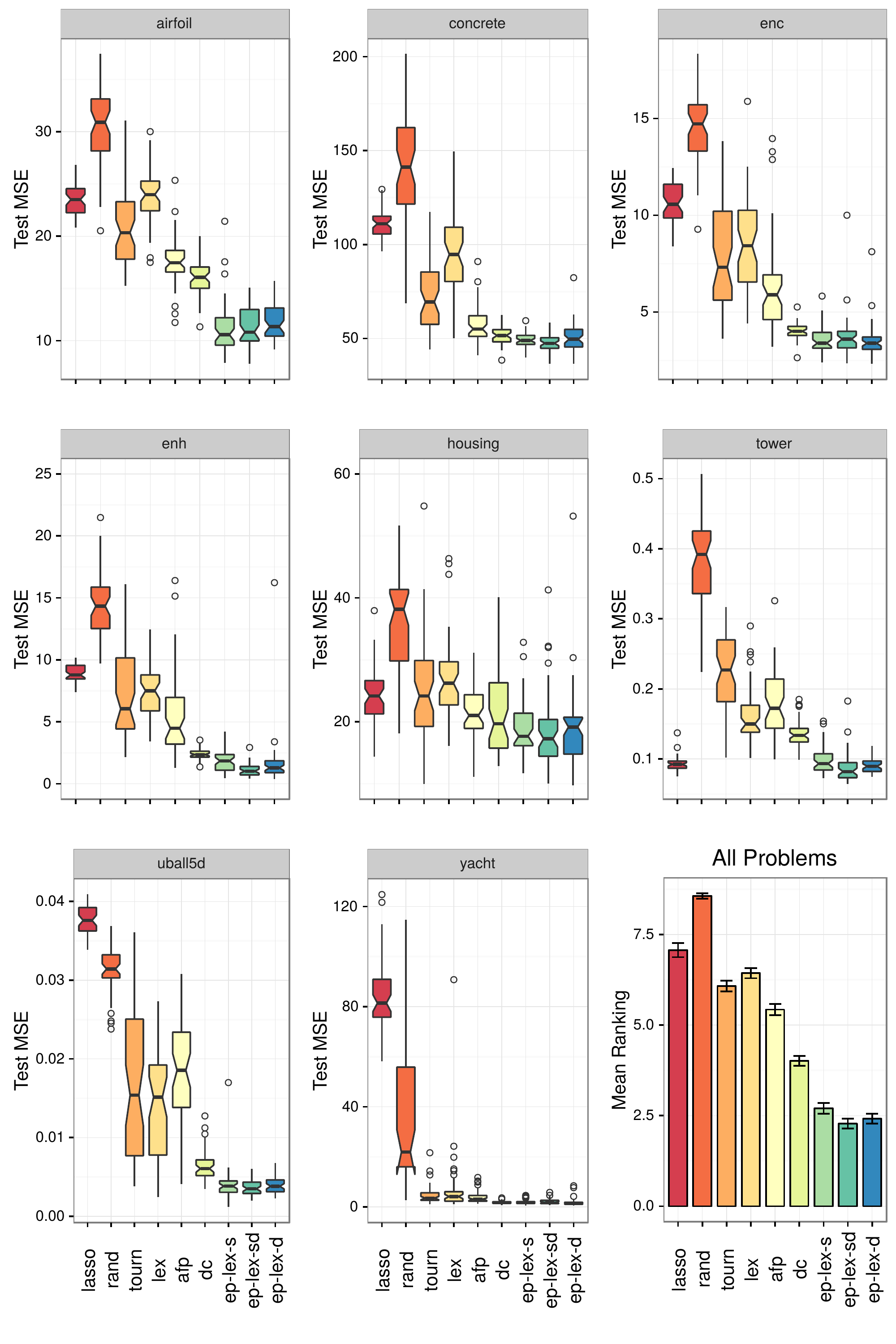}\\
  \caption{Boxplots of the mean squared error on the test set for 50 randomized trials of each algorithm on the regression benchmark datasets.}\label{fig:boxplot_reg}
\end{figure}

\begin{figure}
\centering
  \includegraphics[width=\textwidth]{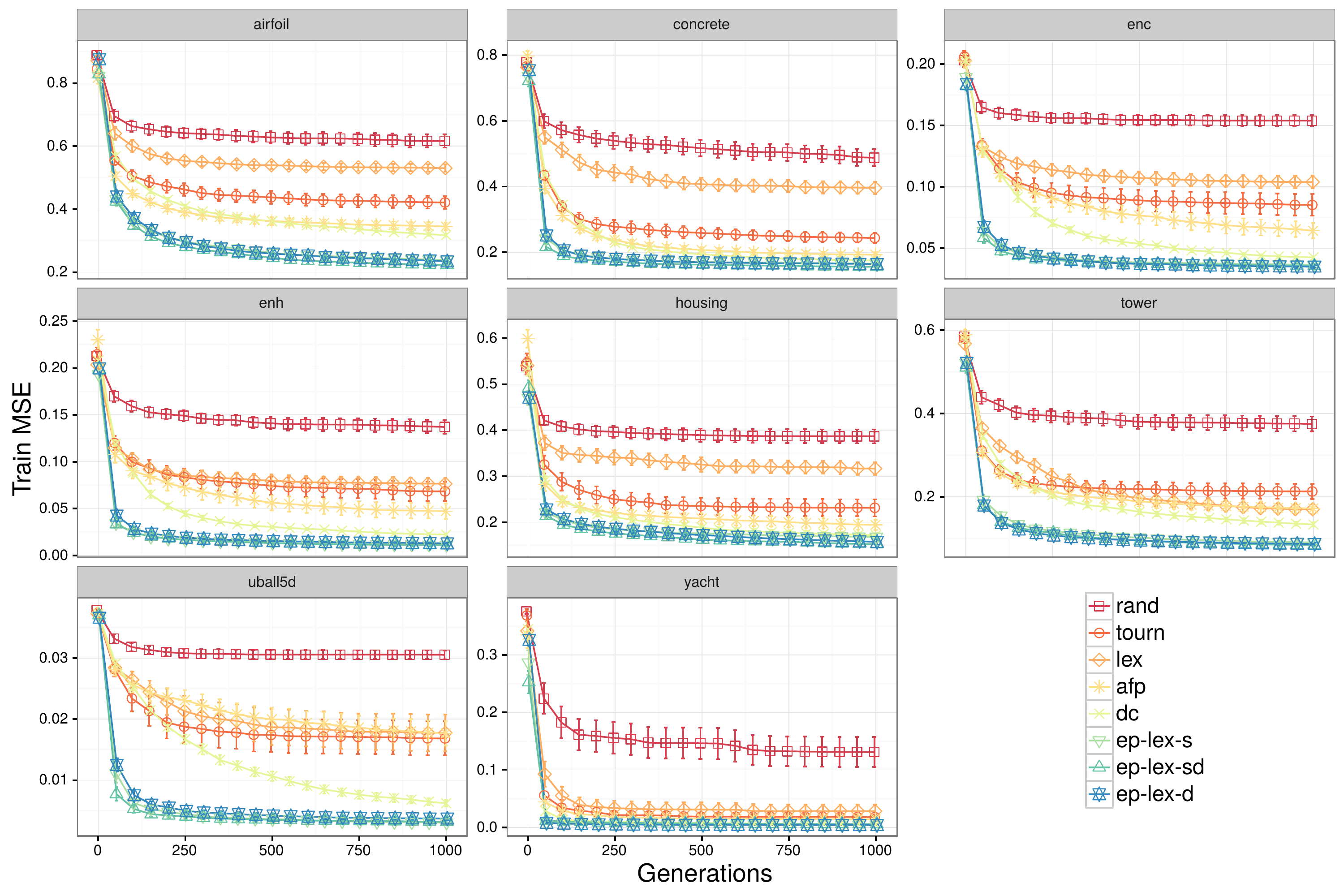}\\
 \caption{Training error for the best individual using different selection methods. The results are averaged over 50 trials with 95\% confidence intervals.}
\label{fig:train}
\end{figure}

\begin{figure}
\centering
  \includegraphics[width=\textwidth]{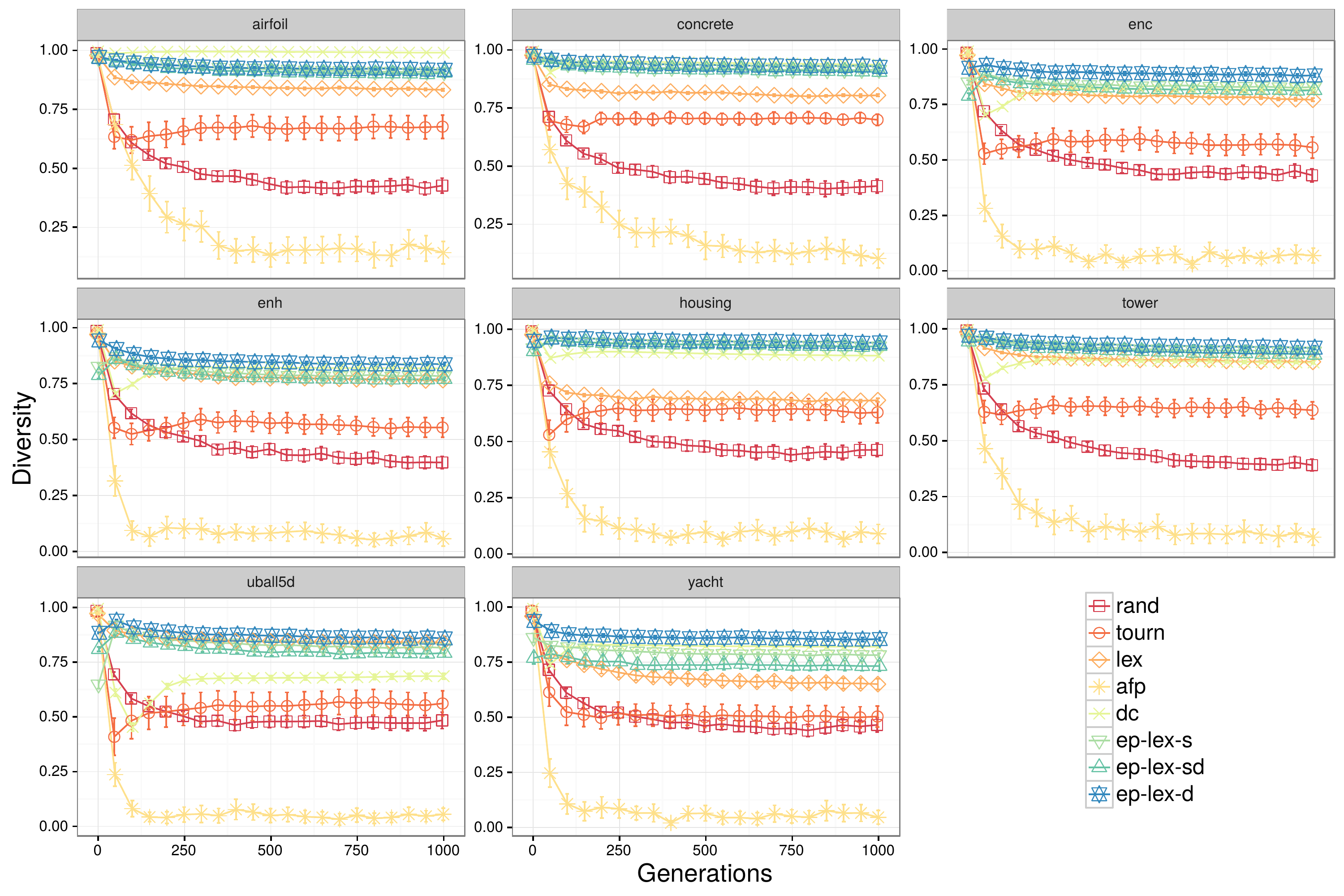}\\
 \caption{Behavioral diversity of the population using different selection methods. The results are averaged over 50 trials with 95\% confidence intervals.}
\label{fig:diversity}
\end{figure}

\begin{figure}[htb]
\centering
  \includegraphics[width=\textwidth]{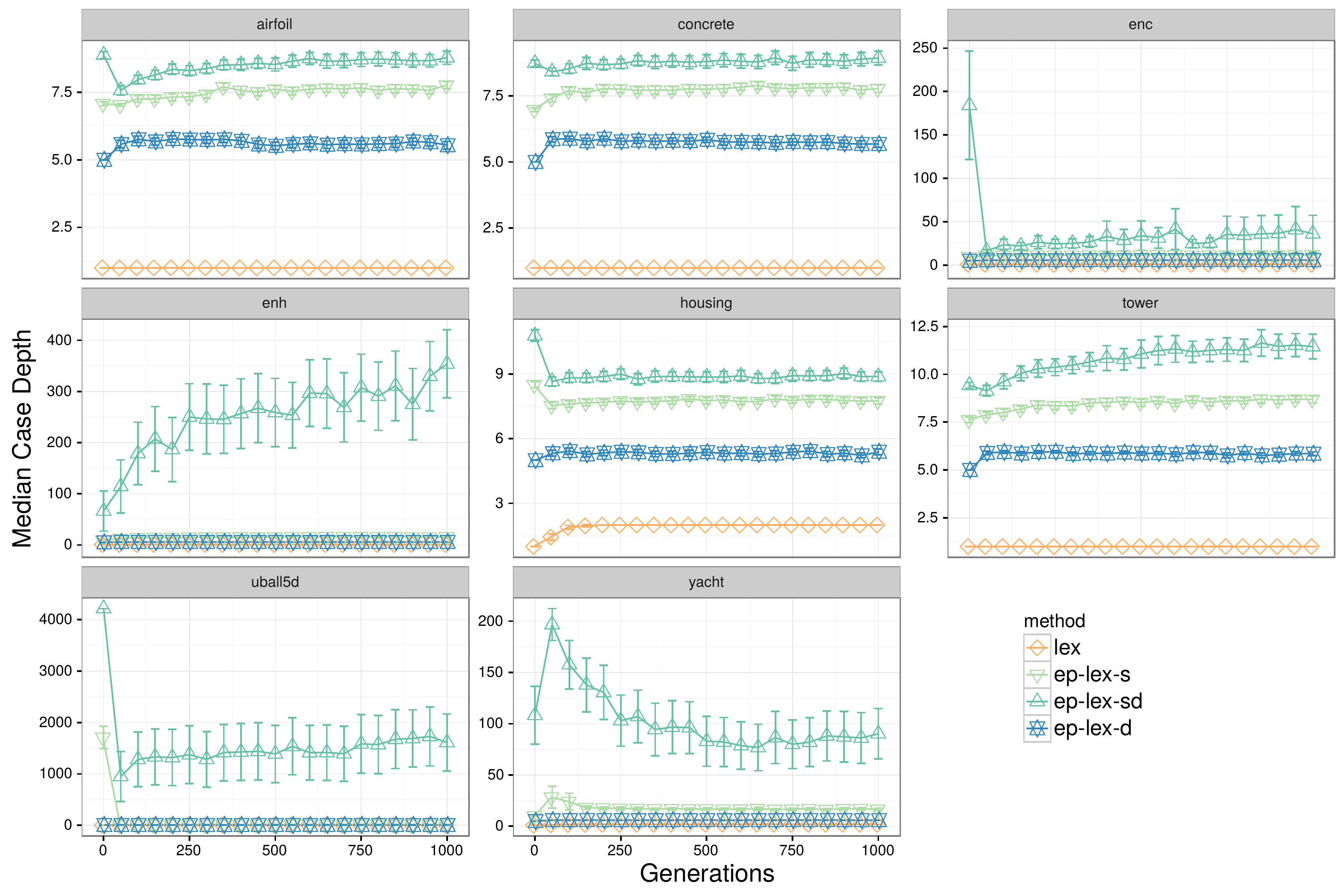}\\
 \caption{Median case depths of selection each generation for the lexicase selection variants on the regression problems. The results are averaged over 50 trials with 95\% confidence intervals. } 
\label{fig:case_depth}
\end{figure}

\label{tbl:tables===}

% latex table generated in R 3.2.3 by xtable 1.8-2 package
% Mon Apr 24 15:11:52 2017
\begin{table}[ht]
\centering
\caption{Significance test $p$-values comparing test MSE using the pair-wise Wilcoxon rank-sum test with Holm correction for multiple comparisons. All significance tests are conducted relative to semi-dynamic $\epsilon$-lexicase (ep-lex-sd). Bold indicates $p<$ 0.05.} 
\label{tbl:wilcox}
\begingroup\footnotesize
\begin{tabular}{rllllllll}
  \toprule
 & lasso & rand & tourn & lex & afp & dc & ep-lex-s & ep-lex-d \\ 
  \midrule
airfoil & {\bf 2.54e-16} & {\bf 2.54e-16} & {\bf 2.54e-16} & {\bf 2.54e-16} & {\bf 2.55e-15} & {\bf 1.59e-14} & 0.57 & 0.57 \\ 
  concrete & {\bf 2.54e-16} & {\bf 2.54e-16} & {\bf 6.24e-13} & {\bf 4.25e-16} & {\bf 2.74e-08} & {\bf 1.66e-04} & 0.1 & 0.057 \\ 
  enc & {\bf 5.15e-16} & {\bf 2.54e-16} & {\bf 4.12e-14} & {\bf 2.57e-15} & {\bf 1.67e-12} & {\bf 1.61e-03} &   1 & 0.49 \\ 
  enh & {\bf 2.54e-16} & {\bf 2.54e-16} & {\bf 2.67e-16} & {\bf 2.54e-16} & {\bf 1.41e-15} & {\bf 2.00e-14} & {\bf 1.21e-04} & {\bf 1.28e-02} \\ 
  housing & {\bf 1.51e-05} & {\bf 6.20e-13} & {\bf 8.12e-04} & {\bf 3.40e-07} & {\bf 1.57e-02} & 0.22 &   1 &   1 \\ 
  tower & {\bf 6.38e-03} & {\bf 2.54e-16} & {\bf 1.57e-15} & {\bf 6.39e-15} & {\bf 7.63e-15} & {\bf 3.67e-14} & {\bf 6.38e-03} & 0.066 \\ 
  uball5d & {\bf 2.54e-16} & {\bf 2.54e-16} & {\bf 4.80e-15} & {\bf 1.04e-13} & {\bf 6.96e-16} & {\bf 1.55e-11} &   1 &   1 \\ 
  yacht & {\bf 2.54e-16} & {\bf 5.46e-16} & {\bf 1.52e-07} & {\bf 7.86e-07} & {\bf 4.93e-06} &   1 &   1 & 0.053 \\ 
   \bottomrule
\end{tabular}
\endgroup
\end{table}

% latex table generated in R 3.2.3 by xtable 1.8-2 package
% Mon Apr 24 15:12:21 2017
\begin{table}[ht]
\centering
\caption{Post-hoc pairwise statistical tests of rankings across problems according to Tukey's Honest Significant Difference test. Bold values indicate $p<$ 0.05 with adjustment for multiple comparisons.} 
\label{tbl:hsd}
\begingroup\footnotesize
\begin{tabular}{rllllllll}
  \toprule
 & lasso & rand & tourn & lex & afp & dc & ep-lex-s & ep-lex-sd \\ 
  \midrule
ep-lex-s & {\bf 1.55e-11} & {\bf 1.53e-11} & {\bf 1.36e-09} & {\bf 6.19e-11} & {\bf 6.32e-07} & 0.066 &  &  \\ 
  ep-lex-sd & {\bf 1.54e-11} & {\bf 1.53e-11} & {\bf 4.00e-11} & {\bf 1.63e-11} & {\bf 1.17e-08} & {\bf 3.59e-03} & 0.98 &  \\ 
  ep-lex-d & {\bf 1.54e-11} & {\bf 1.53e-11} & {\bf 1.05e-10} & {\bf 1.86e-11} & {\bf 4.32e-08} & {\bf 1.00e-02} &   1 &   1 \\ 
   \bottomrule
\end{tabular}
\endgroup
\end{table}

\subsection{Scaling Experiment}
In order to get an empirical sense of the time scaling of $\epsilon$-lexicase selection in comparison to other selection methods, we run a set of experiments in which the population size is varied between 50 and 2000 while using a fixed training set of 100 samples from the UBall5D problem. We run 10 trials of each population size setting and compare the eight GP methods listed above. We use the results to estimate the time complexity of the $\epsilon$-lexicase selection variants as a function of population size. 

\subsubsection{Scaling Results}

The results of the time complexity experiment are shown in Figure~\ref{fig:time_pop} as a log-log plot with wall-clock times on the y-axis and the population size on the x-axis. We estimate the time scaling as a function of population size by fitting a linear model to the log-transformed results, as $\log(\text{Runtime}(N)) = a + b \log(N)$, which gives $\text{Runtime}(N) = a N^b$. The linear models are shown in Figure~\ref{fig:time_pop} for the $\epsilon$-lexicase selection methods, which estimate the exponent of the complexity model, $b$, to be between 0.935 and 0.944. Therefore on average over these settings, the runtime of $\epsilon$-lexicase selection as a function of $N$ is approximately $\text{Runtime}(N) = 0.45 N^{0.939}$. This suggests a much lower time complexity with respect to $N$ in practice than the worst-case complexity of $N^2$ (see \S\ref{s:lex}). In general, the lexicase methods fall between deterministic crowding and tournament selection in terms of wall clock times, with afp achieving the lowest times at higher population sizes. All runtime differences between methods are well within an order of magnitude. 

\begin{figure}[htb]
%\begin{minipage}{0.49\textwidth}
\centering
  \includegraphics[width=0.6\textwidth]{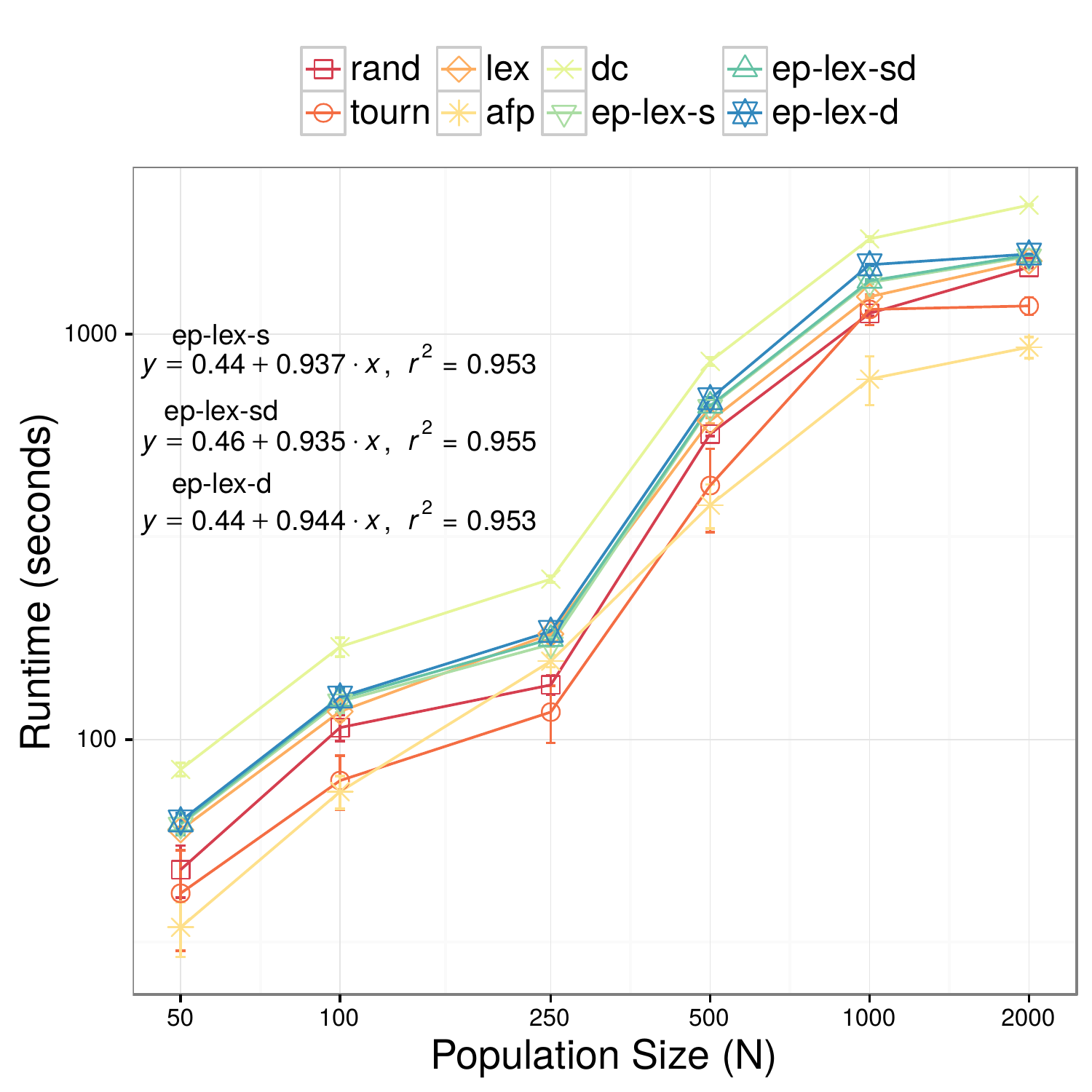}\\
 \caption{Scaling of wall-clock runtimes as a function of population size.}\label{fig:time_pop}
% \end{minipage}
% \hspace{0.02\textwidth}
% \begin{minipage}{0.49\textwidth}
% \centering
%  \includegraphics[width=\textwidth]{figs/regression_samples_scaling.pdf}\\
% \caption{Scaling of wall-clock runtimes as a function of number of training cases.}\label{fig:time_samples}

% \end{minipage}
%\end{figure}
%\begin{figure}

\end{figure}

\section{Discussion}\label{s:discuss}
The experimental results show that $\epsilon$-lexicase selection performs well on the symbolic regression problems compared to other \edit{GP} methods \edit{and Lasso}. $\epsilon$-lexicase leads to quicker learning on the training set (Figure~\ref{fig:train}) and better test set performance (Figure~\ref{fig:boxplot_reg}) \edit{than other GP methods}. The improvement in performance compared to traditional selection methods appears to be tied to the high semantic diversity that $\epsilon$-lexicase selection maintains throughout training (Figure~\ref{fig:diversity}), and its preservation of individuals that perform well on unique portions of the training cases. $\epsilon$-lexicase selection shows a categorical improvement over lexicase selection for these continuous valued problems. Although lexicase selection also maintains diverse semantics, its inferior performance can be explained by its under-utilization of training cases for selection (Figure~\ref{fig:case_depth}) and its property of selecting only among strictly elite individuals (see the example from \S\ref{s:ex}), a property that is relaxed through the introduction of $\epsilon$ thresholds in $\epsilon$-lexicase selection.   

Two new variants of $\epsilon$-lexicase selection, semi-dynamic and dynamic, perform the best overall in our experiments. However, the variants of $\epsilon$-lexicase do not differ significantly across all tested problems, which suggests that the foundations of the method are robust to different definitions of $\epsilon$ as long as they result in higher leverage of case information during selection compared to normal lexicase selection, which underperforms on regression problems. In view of the results, we suggest semi-dynamic $\epsilon$-lexicase (ep-lex-sd, Algorithm~\ref{alg:ep-lex-sd}) as the default implementation of $\epsilon$-lexicase selection since it has the lowest mean test ranking and appears to utilize the most case information according to Figure~\ref{fig:case_depth}. 

$\epsilon$-lexicase selection is a global pool, uniform random sequence, non-elitist version of lexicase selection~\citep{spector_assessment_2012}. Compared to traditional lexicase selection, which is elitist, $\epsilon$-lexicase selection represents a relaxed version of lexicase selection; other potential relaxations could show similar benefits. ``Global pool" means that each selection event begins with the entire population; however it is possible that smaller pools, perhaps defined geographically~\citep{spector_trivial_2006}, could improve performance on certain problems that respond well to relaxed selection pressure. Future work could also consider alternative orderings of test cases that may perform better than ``uniform random sequence" ordering that has been the focus of work thus far. \cite{liskowski_comparison_2015} attempted to use derived objective clusters as cases in lexicase selection, but found that this actually decreased performance, possibly due to the small number of resultant objectives. \cite{burks_investigation_2016} found biasing case orderings in terms of performance yielded mixed results. Nevertheless, there may be a form of ordering or case reduction that improves lexicase selection's performance over random shuffling.   
 
The ordering of the training cases that produce a given parent also contains potentially useful information about the parent that could be used by the search operators in GP. \cite{helmuth_general_2015-1} observed that lexicase selection creates large numbers of distinct behavioral clusters in the population (an observation supported by Figure~\ref{fig:diversity}). In that regard, it may be advantageous, for instance, to perform crossover on individuals selected by differing orders of cases such that their offspring are more likely to inherit subprograms with unique partial solutions to a given task. Recent work has highlighted the usefulness of semantically diverse parents when conducting geometric semantic crossover in geometic semantic GP~\citep{chen_geometric_2017}.   

Based on the observations in \S\ref{s:prob}, when the training set is much larger than the population size, some cases are likely to go unused. In these scenarios it may be advantageous to reduce program evaluations by lazily evaluating programs on cases as they appear in selection events. Indeed, Eqn.~\ref{eq:prob_case} could be used as a guide for determining when a lazy evaluation strategy would lead to significant computational savings.

\edit{Limitations of the current experimental analysis should be noted. First, we have not considered hyperparameter tuning of the GP system, which we intend to pursue in future work. In addition, the non-GP regression comparisons are limited to Lasso. In future work, we intend to compare to a broader class of learning algorithms. Finally, we have considered lexicase and $\epsilon$-lexicase selection only in the context of GP applied to symbolic regression. Future work should consider the application of these selection methods to other areas of EC, and the use of these algorithms for other learning tasks.}

\section{Conclusions}\label{s:conclusion}
In this paper we present a probabilistic and multi-objective analysis of lexicase selection and $\epsilon$-lexicase selection. We develop the expected probabilities of selection under lexicase selection variants, and show the impact of population size and training set size on probabilities of selection. For the first time, the connection between lexicase selection and multi-objective optimization is analyzed, showing that individuals selected by lexicase selection occupy the boundaries or near boundaries of the Pareto front in the high-dimensional space spanned by the population errors. 

In addition, we experimentally validate $\epsilon$-lexicase selection, including the new semi-dynamic and dynamic variants, on a set of real-world and synthetic symbolic regression problems. The results suggest that $\epsilon$-lexicase selection strongly improves the ability of GP to find accurate models. Further analysis of these runs show that lexicase variants maintain exceptionally high diversity during evolution, and that $\epsilon$-lexicase variants consider more cases per selection event than standard lexicase selection. The results validate our motivation for creating this variant of lexicase for continuous domains, and suggest the adoption of lexicase selection and variants of it in similar domains.  

\section{Acknowledgments}
This work was supported by the Warren Center for Network and Data Science at UPenn, NIH grants P30-ES013508, AI116794 and LM009012, and NSF grants 1617087, 1129139 and 1331283. Any opinions, findings, and conclusions or recommendations expressed in this publication are those of the authors and do not necessarily reflect the views of the National Science Foundation.

\bibliographystyle{apalike}
\bibliography{epsilon_lexicase}

\begin{thebibliography}{}

\bibitem[Auger et~al., 2009]{auger_theory_2009}
Auger, A., Bader, J., Brockhoff, D., and Zitzler, E. (2009).
\newblock Theory of the hypervolume indicator: optimal μ-distributions and the
  choice of the reference point.
\newblock In {\em Proceedings of the tenth {ACM} {SIGEVO} workshop on
  {Foundations} of genetic algorithms}, pages 87--102. ACM.

\bibitem[Burks and Punch, 2016]{burks_investigation_2016}
Burks, A.~R. and Punch, W.~F. (2016).
\newblock An investigation of hybrid structural and behavioral diversity
  methods in genetic programming.
\newblock In {\em Genetic {Programming} {Theory} and {Practice}}, number XIV.
  Springer, Ann Arbor, MI.

\bibitem[Chand and Wagner, 2015]{chand_evolutionary_2015}
Chand, S. and Wagner, M. (2015).
\newblock Evolutionary many-objective optimization: {A} quick-start guide.
\newblock {\em Surveys in Operations Research and Management Science},
  20(2):35--42.

\bibitem[Chen et~al., 2017]{chen_geometric_2017}
Chen, Q., Xue, B., Mei, Y., and Zhang, M. (2017).
\newblock Geometric {Semantic} {Crossover} with an {Angle}-{Aware} {Mating}
  {Scheme} in {Genetic} {Programming} for {Symbolic} {Regression}.
\newblock In {\em Genetic {Programming}}, Lecture {Notes} in {Computer}
  {Science}, pages 229--245. Springer, Cham.

\bibitem[Deb et~al., 2005]{deb_evaluating_2005}
Deb, K., Mohan, M., and Mishra, S. (2005).
\newblock Evaluating the $\epsilon$-{Domination} {Based} {Multi}-{Objective}
  {Evolutionary} {Algorithm} for a {Quick} {Computation} of {Pareto}-{Optimal}
  {Solutions}.
\newblock {\em Evolutionary Computation}, 13(4):501--525.

\bibitem[Deb et~al., 2002]{schoenauer_fast_2000}
Deb, K., Pratap, A., Agarwal, S., and Meyarivan, T. (2002).
\newblock A fast and elitist multiobjective genetic algorithm: Nsga-ii.
\newblock {\em IEEE transactions on evolutionary computation}, 6(2):182--197.

\bibitem[Efron et~al., 2004]{efron_least_2004}
Efron, B., Hastie, T., Johnstone, I., Tibshirani, R., and {others} (2004).
\newblock Least angle regression.
\newblock {\em The Annals of statistics}, 32(2):407--499.

\bibitem[Farina and Amato, 2002]{farina_optimal_2002}
Farina, M. and Amato, P. (2002).
\newblock On the optimal solution definition for many-criteria optimization
  problems.
\newblock In {\em Fuzzy {Information} {Processing} {Society}, 2002.
  {Proceedings}. {NAFIPS}. 2002 {Annual} {Meeting} of the {North} {American}},
  pages 233--238. IEEE.

\bibitem[Gathercole and Ross, 1994]{gathercole_dynamic_1994}
Gathercole, C. and Ross, P. (1994).
\newblock Dynamic training subset selection for supervised learning in
  {Genetic} {Programming}.
\newblock In Davidor, Y., Schwefel, H.-P., and Männer, R., editors, {\em
  Parallel {Problem} {Solving} from {Nature} — {PPSN} {III}}, number 866 in
  Lecture {Notes} in {Computer} {Science}, pages 312--321. Springer Berlin
  Heidelberg.
\newblock DOI: 10.1007/3-540-58484-6\_275.

\bibitem[Gon{\c{c}}alves and Silva, 2013]{goncalves_balancing_2013}
Gon{\c{c}}alves, I. and Silva, S. (2013).
\newblock Balancing learning and overfitting in genetic programming with
  interleaved sampling of training data.
\newblock In Krawiec, K., Moraglio, A., Hu, T., Etaner-Uyar, A.~{\c{S}}., and
  Hu, B., editors, {\em Genetic Programming: 16th European Conference, EuroGP
  2013, Vienna, Austria, April 3-5, 2013. Proceedings}, pages 73--84, Berlin,
  Heidelberg. Springer Berlin Heidelberg.

\bibitem[Helmuth, 2015]{helmuth_general_2015}
Helmuth, T. (2015).
\newblock {\em General {Program} {Synthesis} from {Examples} {Using} {Genetic}
  {Programming} with {Parent} {Selection} {Based} on {Random} {Lexicographic}
  {Orderings} of {Test} {Cases}}.
\newblock PhD thesis, University of Massachusetts Amherst.

\bibitem[Helmuth et~al., 2016a]{helmuth_effects_2016}
Helmuth, T., McPhee, N.~F., and Spector, L. (2016a).
\newblock Effects of {Lexicase} and {Tournament} {Selection} on {Diversity}
  {Recovery} and {Maintenance}.
\newblock In {\em Companion Proceedings of the 2016 Conference on {Genetic} and
  {Evolutionary} {Computation}}, pages 983--990. ACM.

\bibitem[Helmuth et~al., 2016b]{helmuth_impact_2016}
Helmuth, T., McPhee, N.~F., and Spector, L. (2016b).
\newblock The impact of hyperselection on lexicase selection.
\newblock In {\em Proceedings of the 2016 Conference on {Genetic} and
  {Evolutionary} {Computation}}, pages 717--724. ACM.

\bibitem[Helmuth and Spector, 2015]{helmuth_general_2015-1}
Helmuth, T. and Spector, L. (2015).
\newblock General {Program} {Synthesis} {Benchmark} {Suite}.
\newblock In {\em GECCO '15: Proceedings of the 2015 Conference on Genetic and
  Evolutionary Computation}, pages 1039--1046. ACM Press.

\bibitem[Helmuth et~al., 2014]{helmuth_solving_2014}
Helmuth, T., Spector, L., and Matheson, J. (2014).
\newblock Solving {Uncompromising} {Problems} with {Lexicase} {Selection}.
\newblock {\em IEEE Transactions on Evolutionary Computation}, PP(99):1--1.

\bibitem[Hern\'{a}ndez-D\'{i}az et~al.,
  2007]{hernandez-diaz_pareto-adaptive_2007}
Hern\'{a}ndez-D\'{i}az, A.~G., Santana-Quintero, L.~V., Coello, C. A.~C., and
  Molina, J. (2007).
\newblock Pareto-adaptive ε-dominance.
\newblock {\em Evolutionary computation}, 15(4):493--517.

\bibitem[Ishibuchi et~al., 2008]{ishibuchi_evolutionary_2008}
Ishibuchi, H., Tsukamoto, N., and Nojima, Y. (2008).
\newblock Evolutionary many-objective optimization: {A} short review.
\newblock In {\em {IEEE} congress on evolutionary computation}, pages
  2419--2426. Citeseer.

\bibitem[Klein and Spector, 2008]{klein_genetic_2008}
Klein, J. and Spector, L. (2008).
\newblock Genetic programming with historically assessed hardness.
\newblock {\em Genetic Programming Theory and Practice VI}, pages 61--75.

\bibitem[Krawiec, 2016]{krawiec_behavioral_2016}
Krawiec, K. (2016).
\newblock {\em Behavioral program synthesis with genetic programming}, volume
  618.
\newblock Springer.

\bibitem[Krawiec and Lichocki, 2010]{schaefer_using_2010}
Krawiec, K. and Lichocki, P. (2010).
\newblock Using {Co}-solvability to {Model} and {Exploit} {Synergetic}
  {Effects} in {Evolution}.
\newblock In Schaefer, R., Cotta, C., Kołodziej, J., and Rudolph, G., editors,
  {\em Parallel {Problem} {Solving} from {Nature}, {PPSN} {XI}}, pages
  492--501. Springer Berlin Heidelberg, Berlin, Heidelberg.

\bibitem[Krawiec and Liskowski, 2015]{krawiec_automatic_2015}
Krawiec, K. and Liskowski, P. (2015).
\newblock Automatic derivation of search objectives for test-based genetic
  programming.
\newblock In {\em Genetic {Programming}}, Lecture Notes in Computer Science,
  pages 53--65. Springer.

\bibitem[Krawiec and Nawrocki, 2013]{krawiec_implicit_2013}
Krawiec, K. and Nawrocki, M. (2013).
\newblock Implicit fitness sharing for evolutionary synthesis of license plate
  detectors.
\newblock In Esparcia-Alc{\'a}zar, A.~I., editor, {\em Applications of
  Evolutionary Computation: 16th European Conference, EvoApplications 2013,
  Vienna, Austria, April 3-5, 2013. Proceedings}, Lecture Notes in Computer
  Science, pages 376--386, Berlin, Heidelberg. Springer Berlin Heidelberg.

\bibitem[Krawiec and O'Reilly, 2014]{krawiec_behavioral_2014}
Krawiec, K. and O'Reilly, U.-M. (2014).
\newblock Behavioral programming: a broader and more detailed take on semantic
  {GP}.
\newblock In {\em Proceedings of the 2014 conference on {Genetic} and
  evolutionary computation}, pages 935--942. ACM Press.

\bibitem[La~Cava and Moore, 2017a]{la_cava_general_2017}
La~Cava, W. and Moore, J. (2017a).
\newblock A {General} {Feature} {Engineering} {Wrapper} for {Machine}
  {Learning} {Using} $\epsilon$-{Lexicase} {Survival}.
\newblock In {\em Genetic {Programming}}, Lecture Notes in Computer Science,
  pages 80--95. Springer, Cham.
\newblock DOI: 10.1007/978-3-319-55696-3\_6.

\bibitem[La~Cava and Moore, 2017b]{la_cava_ensemble_2017}
La~Cava, W. and Moore, J.~H. (2017b).
\newblock Ensemble representation learning: an analysis of fitness and survival
  for wrapper-based genetic programming methods.
\newblock In {\em {GECCO} '17: {Proceedings} of the 2017 {Genetic} and
  {Evolutionary} {Computation} {Conference}}, pages 961--968, Berlin, Germany.
  ACM.

\bibitem[La~Cava et~al., 2016]{la_cava_epsilon-lexicase_2016}
La~Cava, W., Spector, L., and Danai, K. (2016).
\newblock Epsilon-{Lexicase} {Selection} for {Regression}.
\newblock In {\em Proceedings of the {Genetic} and {Evolutionary} {Computation}
  {Conference} 2016}, {GECCO} '16, pages 741--748, New York, NY, USA. ACM.

\bibitem[Langdon, 1995]{langdon_evolving_1995}
Langdon, W.~B. (1995).
\newblock Evolving {Data} {Structures} with {Genetic} {Programming}.
\newblock In {\em {ICGA}}, pages 295--302.

\bibitem[Laumanns et~al., 2002]{laumanns_archiving_2002}
Laumanns, M., Thiele, L., Deb, K., and Zitzler, E. (2002).
\newblock Combining convergence and diversity in evolutionary multiobjective
  optimization.
\newblock {\em Evolutionary computation}, 10(3):263--282.

\bibitem[Li et~al., 2015]{li_many-objective_2015}
Li, B., Li, J., Tang, K., and Yao, X. (2015).
\newblock Many-{Objective} {Evolutionary} {Algorithms}: {A} {Survey}.
\newblock {\em ACM Computing Surveys}, 48(1):1--35.

\bibitem[Li et~al., 2017]{li_empirical_2017}
Li, K., Deb, K., Altinoz, T., and Yao, X. (2017).
\newblock Empirical investigations of reference point based methods when facing
  a massively large number of objectives: {First} results.
\newblock In {\em International {Conference} on {Evolutionary}
  {Multi}-{Criterion} {Optimization}}, pages 390--405. Springer.

\bibitem[Li and Zheng, 2009]{li_spread_2009}
Li, M. and Zheng, J. (2009).
\newblock Spread assessment for evolutionary multi-objective optimization.
\newblock In {\em International {Conference} on {Evolutionary}
  {Multi}-{Criterion} {Optimization}}, pages 216--230. Springer.

\bibitem[Lichman, 2013]{lichman_uci_2013}
Lichman, M. (2013).
\newblock {\em {UCI} {Machine} {Learning} {Repository}}.
\newblock University of California, Irvine, School of Information and Computer
  Sciences.

\bibitem[Liskowski and Krawiec, 2017]{liskowski_discovery_2017}
Liskowski, P. and Krawiec, K. (2017).
\newblock Discovery of {Search} {Objectives} in {Continuous} {Domains}.
\newblock In {\em Proceedings of the {Genetic} and {Evolutionary} {Computation}
  {Conference}}, {GECCO} '17, pages 969--976, New York, NY, USA. ACM.

\bibitem[Liskowski et~al., 2015]{liskowski_comparison_2015}
Liskowski, P., Krawiec, K., Helmuth, T., and Spector, L. (2015).
\newblock Comparison of {Semantic}-aware {Selection} {Methods} in {Genetic}
  {Programming}.
\newblock In {\em Proceedings of the {Companion} {Publication} of the 2015
  {Annual} {Conference} on {Genetic} and {Evolutionary} {Computation}}, {GECCO}
  {Companion} '15, pages 1301--1307, New York, NY, USA. ACM.

\bibitem[Mahfoud, 1995]{mahfoud_niching_1995}
Mahfoud, S.~W. (1995).
\newblock {\em Niching methods for genetic algorithms}.
\newblock PhD thesis, University of Illinois at Urbana-Champaign.

\bibitem[Mart\'{i}nez et~al., 2013]{martinez_searching_2013}
Mart\'{i}nez, Y., Naredo, E., Trujillo, L., and Galv\'{a}n-L\'{o}pez, E.
  (2013).
\newblock Searching for novel regression functions.
\newblock In {\em Evolutionary {Computation} ({CEC}), 2013 {IEEE} {Congress}
  on}, pages 16--23. IEEE.

\bibitem[McKay, 2001]{mckay_investigation_2001}
McKay, R. I.~B. (2001).
\newblock An {Investigation} of {Fitness} {Sharing} in {Genetic} {Programming}.
\newblock {\em The Australian Journal of Intelligent Information Processing
  Systems}, 7(1/2):43--51.

\bibitem[Moraglio et~al., 2012]{moraglio_geometric_2012}
Moraglio, A., Krawiec, K., and Johnson, C.~G. (2012).
\newblock Geometric semantic genetic programming.
\newblock In {\em Parallel {Problem} {Solving} from {Nature}-{PPSN} {XII}},
  pages 21--31. Springer.

\bibitem[Pham-Gia and Hung, 2001]{pham-gia_mean_2001}
Pham-Gia, T. and Hung, T.~L. (2001).
\newblock The mean and median absolute deviations.
\newblock {\em Mathematical and Computer Modelling}, 34(7–8):921--936.

\bibitem[Poli and Langdon, 1998]{poli_schema_1998}
Poli, R. and Langdon, W.~B. (1998).
\newblock Schema theory for genetic programming with one-point crossover and
  point mutation.
\newblock {\em Evolutionary Computation}, 6(3):231--252.

\bibitem[Schmidt and Lipson, 2008]{schmidt_coevolution_2008}
Schmidt, M. and Lipson, H. (2008).
\newblock Coevolution of {Fitness} {Predictors}.
\newblock {\em IEEE Transactions on Evolutionary Computation}, 12(6):736--749.

\bibitem[Schmidt and Lipson, 2011]{schmidt_age-fitness_2011}
Schmidt, M. and Lipson, H. (2011).
\newblock Age-fitness pareto optimization.
\newblock In {\em Genetic {Programming} {Theory} and {Practice} {VIII}}, pages
  129--146. Springer.

\bibitem[Smits and Kotanchek, 2005]{smits_pareto-front_2005}
Smits, G.~F. and Kotanchek, M. (2005).
\newblock Pareto-front exploitation in symbolic regression.
\newblock In {\em Genetic {Programming} {Theory} and {Practice} {II}}, pages
  283--299. Springer.

\bibitem[Spector, 2012]{spector_assessment_2012}
Spector, L. (2012).
\newblock Assessment of problem modality by differential performance of
  lexicase selection in genetic programming: a preliminary report.
\newblock In {\em Proceedings of the fourteenth international conference on
  {Genetic} and evolutionary computation conference companion}, pages 401--408.

\bibitem[Spector and Klein, 2006]{spector_trivial_2006}
Spector, L. and Klein, J. (2006).
\newblock Trivial geography in genetic programming.
\newblock In {\em Genetic programming theory and practice {III}}, pages
  109--123. Springer.

\bibitem[Tibshirani, 1996]{tibshirani_regression_1996}
Tibshirani, R. (1996).
\newblock Regression shrinkage and selection via the lasso.
\newblock {\em Journal of the Royal Statistical Society. Series B
  (Methodological)}, pages 267--288.

\bibitem[Vanneschi et~al., 2014]{vanneschi_survey_2014}
Vanneschi, L., Castelli, M., and Silva, S. (2014).
\newblock A survey of semantic methods in genetic programming.
\newblock {\em Genetic Programming and Evolvable Machines}, 15(2):195--214.

\bibitem[Wagner et~al., 2007]{wagner_pareto-_2007}
Wagner, T., Beume, N., and Naujoks, B. (2007).
\newblock Pareto-, {Aggregation}-, and {Indicator}-{Based} {Methods} in
  {Many}-{Objective} {Optimization}.
\newblock In {\em Evolutionary {Multi}-{Criterion} {Optimization}}, pages
  742--756. Springer, Berlin, Heidelberg.
\newblock DOI: 10.1007/978-3-540-70928-2\_56.

\bibitem[White et~al., 2012]{white_better_2012}
White, D.~R., McDermott, J., Castelli, M., Manzoni, L., Goldman, B.~W.,
  Kronberger, G., Jaśkowski, W., O'Reilly, U.-M., and Luke, S. (2012).
\newblock Better {GP} benchmarks: community survey results and proposals.
\newblock {\em Genetic Programming and Evolvable Machines}, 14(1):3--29.

\bibitem[Xie and Zhang, 2013]{xie_parent_2013}
Xie, H. and Zhang, M. (2013).
\newblock Parent {Selection} {Pressure} {Auto}-{Tuning} for {Tournament}
  {Selection} in {Genetic} {Programming}.
\newblock {\em IEEE Transactions on Evolutionary Computation}, 17(1):1--19.

\bibitem[Xie et~al., 2007]{xie_another_2007}
Xie, H., Zhang, M., and Andreae, P. (2007).
\newblock Another investigation on tournament selection: modelling and
  visualisation.
\newblock In {\em Proceedings of the 9th annual conference on {Genetic} and
  evolutionary computation}, pages 1468--1475. ACM.

\bibitem[Zitzler et~al., 2001]{zitzler_spea2:_2001}
Zitzler, E., Laumanns, M., and Thiele, L. (2001).
\newblock {SPEA}2: {Improving} the strength {Pareto} evolutionary algorithm.
\newblock In K.~Giannakoglou, D.~Tsahalis, J. P. P.~P. and Fogarty, T.,
  editors, {\em EUROGEN 2001}, Evolutionary Methods for Design, Optimization
  and Control with Applications to Industrial Problems, pages 95--100, Athens,
  Greece.

\end{thebibliography}

\end{document}